\documentclass[twoside]{article}
\usepackage{lmodern}

\usepackage[preprint]{aistats2027}
  \makeatletter\let\ps@empty\ps@plain \makeatother

\makeatletter\@secpenalty=0 \makeatother

\usepackage{fontawesome5}
\usepackage{hyperref}
\usepackage{amsmath,amssymb,amsthm}
\usepackage{booktabs}
\usepackage{graphicx}
\usepackage{subcaption}
\usepackage{algorithm}
\usepackage{algpseudocode}
\usepackage[font=small,labelfont=bf,skip=4pt]{caption}
\usepackage{xcolor}
\usepackage{hyperref}
\usepackage{url}
\usepackage{microtype}
\usepackage[authoryear,round]{natbib}

\graphicspath{{../code/figures/out/}{../code/}}

\hypersetup{colorlinks=true,linkcolor=blue!50!black,citecolor=blue!50!black,urlcolor=magenta}

\theoremstyle{plain}
\newtheorem{theorem}{Theorem}
\newtheorem{proposition}[theorem]{Proposition}

\newtheorem{corollary}[theorem]{Corollary}
\theoremstyle{definition}
\newtheorem{definition}[theorem]{Definition}
\newtheorem{assumption}[theorem]{Assumption}

\theoremstyle{remark}
\newtheorem{remark}[theorem]{Remark}

\newcommand{\Prob}{\mathbb{P}}
\newcommand{\Hist}{\mathcal{H}}

\newcommand{\eqpi}{\sim_{\pi}}

\begin{document}

\runningtitle{Exact Distinguishability in Non-Markovian Decision Processes}

\twocolumn[

\aistatstitle{Exact Distinguishability in Non-Markovian Decision Processes}

\vspace{-12pt}
\aistatsauthor{Kabir Murjani \And Nisarg Patel }

\aistatsaddress{Nirma University \\ Ahmedabad, India \\ \texttt{kabirmurjani@gmail.com} \And
  Google \\ Mountain View, CA \\ \texttt{nisargpatel@google.com} }
 ]

\begin{abstract}
Non-Markovian environments are often modeled as Regular Decision Processes (RDPs), where dynamics depend on the interaction history through a finite automaton. Existing offline guarantees for RDPs rely on a distinguishability assumption on the behaviour policy but provide no means of verifying it. When the assumption is violated, distinct models may explain the data equally well. We study when data collected under a fixed behaviour policy can distinguish two candidate RDPs. We prove that the posterior odds between observationally equivalent candidates remain equal to the prior odds at every sample size, even when the policy visits every automaton state, and verify both results formally in Lean 4. We then characterize this equivalence exactly and derive PEC, an algorithm that decides it in time linear in the size of the product automaton. The distinguishability assumption of prior work fails on three of our four test environments, and the experiment identified by PEC restores it in each case.\footnote{\faGithub\ \href{https://github.com/Kcbir/pec}{Source code and Lean 4 formalisations.}}
\end{abstract}

\section{Introduction}
\label{sec:intro}

To model non-Markovian environments, recent work frequently utilizes the Regular Decision Process (RDP) \citep{brafman2019rdp}, a framework where the next observation depends on the entire interaction history through the state of a finite automaton driven by the action-observation sequence.  The model class sits between Markov decision processes (MDPs) and partially observable Markov decision processes (POMDPs), and a recent line of work learns it offline, from a corpus of interaction data produced under a fixed operating policy.

A corpus of that
kind is consistent with several environments at once. A point estimator fitted to it returns
one of them with no indication that the others fit equally well, and a Bayesian learner's odds
between them remain at the prior, which from inside a single run cannot be told apart from slow
learning. Existing guarantees for this model class presuppose that the corpus is informative.
\citet{regorl} gave the first provably
efficient offline algorithm for RDPs, RegORL, with probably approximately correct (PAC)
guarantees under a distinguishability assumption on the behaviour policy, and
\citet{iclr2025rdp} replace that parameter, which decays exponentially in the number of
automaton states, with a formal-language metric. Neither gives a way to check whether a given
policy satisfies the assumption. Consider a freight planner who reroutes around a lane for
one period after each late delivery on it. The records contain the lane and contain late
deliveries, yet a model in which lateness persists into the next period and one in which it
does not agree on every shipment the rule sends. No quantity of such records therefore tests
the belief that motivated the rule, although a single shipment sent on the lane just after a
late delivery would test it. Given two candidate models and a behaviour policy, we investigate whether any quantity of data collected under that policy can distinguish between them, and if not, we identify which experiment would separate them.

The question is posed for a stated pair. Such pairs arise as competing mechanistic accounts
already written down in a domain, as a hypothesis tested against its memoryless null, or as
the top modes of a Bayesian structure learner compared pairwise. Deciding identifiability over
an implicitly specified hypothesis class is a different problem and remains open
(\S\ref{sec:limits}).

Whether data collected under the behaviour policy can separate two candidates is determined by
the candidates and the policy alone. We decide it exactly, before any learner is
fitted, and the answer applies to a corpus of any size. The relevant notion is $\pi$-observational equivalence: two candidates that
assign equal likelihood to every history the behaviour policy can generate leave the posterior
odds between them at the prior odds at every sample size (Theorem~\ref{thm:frozen}). The
condition differs from state coverage. On a two-action T-maze
the behaviour policy visits every state while the posterior never moves
(Proposition~\ref{prop:collapse}), and both results are checked in Lean~4 against
Mathlib (\S\ref{sec:frozen}).

Characterizing $\pi$-equivalence exactly depends on choosing the right object
(\S\ref{sec:measurezero}). The set of histories on which two likelihoods differ need not be
regular, since two products of real numbers may coincide arithmetically, so it admits no
decision procedure of the usual kind. Posing the question instead over the disagreeing
state--action pairs of the product automaton yields a characterization that holds for an
arbitrary behaviour policy, at every finite horizon and at infinite horizon
(Theorem~\ref{thm:measurezero}). When the action support of the policy does not vary with
history, that characterization reduces to reachability in the product automaton, and the
resulting procedure, \textsc{Pec} (the $\pi$-equivalence certifier), decides it in time
$O(|Q|\,|Q'|\,|A|\,|O|)$, independently of the horizon (\S\ref{sec:decide}). When the
candidates are not equivalent, \textsc{Pec} returns the state pair and the action at which the
corpus must be extended, which is the missing experiment itself. It agrees with brute-force
enumeration of histories on all 3700 model pairs we tested, and its running time grows
linearly in the size of the product automaton up to $10^6$ state pairs.

The same search answers the question that prior offline guarantees leave open. The
distinguishability assumption on which the guarantee of \citet{regorl} rests takes time
exponential in the horizon to evaluate from its definition, whereas \textsc{Pec} decides it
without reference to the horizon (\S\ref{sec:regorl}). Computed in exact arithmetic, the
distinguishability parameter is zero on three of our four test environments, each of which
\textsc{Pec} certifies as equivalent, and in every case the experiment named by the witness
restores the assumption.

\section{Related Work}
\label{sec:related}

\paragraph{Non-Markovian formalisms.} Reward machines place history dependence in the reward
alone and leave the dynamics Markov \citep{toroicarte2022rm,pushdownrm}. Regular Decision
Processes \citep{brafman2019rdp} make the dynamics themselves history dependent, and
\citet{abadi2020} gave the first method for learning and solving them. Subsequent work covers
offline reinforcement learning with PAC guarantees
\citep{regorl,iclr2025rdp,tractableoffline}, Markov abstractions
\citep{markovabstractions}, online learning of non-Markovian reward models
\citep{onlinenmrm}, task automata \citep{taskautomata}, and automata cascades
\citep{cascades}. 

\paragraph{Offline guarantees.} \citet{regorl} established
the first provably efficient offline reinforcement learning framework for this model class,
with a sample complexity parameterized by a distinguishability quantity $\mu_0$ that the
corpus is assumed to bound away from zero. That assumption is the natural counterpart of our
question, and we compute their parameter exactly (\S\ref{sec:regorl}). The
two methods return different objects, since RegORL returns a policy with a bound conditional
on identifiability while \textsc{Pec} returns identifiability itself.
Appendix~\ref{sec:rel-offline} discusses that comparison in full.

\paragraph{Automata learning and non-identifiability.} $L^*_{\mathrm{MDP}}$
\citep{tappler2021lstar} learns MDP structure from black-box interaction and already designs
schedulers to disambiguate, and practical frameworks answer equivalence queries by
statistical conformance testing \citep{aalpy}. The barrier in our setting is the absence of a
system to query, since conformance testing still has to run new traces while we hold a
fixed corpus, and the setting of those methods is Markov given the state
\citep{advice2025,aalagenda2026}. Priors over automata exist
\citep{pfau2010pdia,collapsedpfa,deepdfa} but model sequences without actions or a behaviour
policy, and Bayesian treatments of non-Markovian reward keep the history dependence in the
reward \citep{birlnmr,inferprm,pigls}. Freezing under an uninformative likelihood is known in
off-policy evaluation and causal inference
\citep{ope2022causal,worldmodelnonident,ivconfounded}. In a Markov MDP, non-identifiability under a behaviour policy is a state--action
coverage condition on the true model. Expected information gain \citep{bald} has not been combined with priors over
automaton structure, and the theory of contaminated priors and misspecified posteriors that we
rely on is classical
\citep{berk1966,shalizi2009}.

\section{Preliminaries}
\label{sec:defs}
This framework requires no independent reward function, since a reward taking finitely many values is naturally subsumed into $O$ by $O = O_{env} \times R$, with $e(q,a)$ acting as the joint emission over the pair. 

\subsection{Regular decision processes}

Fix finite sets $A$ of actions and $O$ of observations. A history is a finite sequence
$h = (a_1,o_1)\cdots(a_T,o_T) \in (A \times O)^*$, and we write $\Hist = (A\times O)^*$ and
$\Hist_T$ for histories of length $T$. For $1 \le t \le |h|+1$ we write
$h_{<t} = (a_1,o_1)\cdots(a_{t-1},o_{t-1})$ for the prefix of length $t-1$, with
$h_{<1} = \lambda$ the empty history.

\begin{definition}[RDP]
\label{def:rdp}
An RDP is a tuple $M = (Q, q_0, \delta, e)$ where $Q$ is a finite set of automaton states,
$q_0 \in Q$ is initial, $\delta : Q \times A \times O \to Q$ is the transition function, and
$e : Q \times A \to \Delta(O)$ assigns to each state--action pair a distribution over
observations, so that $\sum_{o \in O} e(q,a)(o) = 1$ for every $(q,a)$.
\end{definition}

The automaton state is a deterministic function of history. We extend $\delta$ to
$\hat\delta : \Hist \to Q$ by $\hat\delta(\lambda) = q_0$ and
$\hat\delta(h\,(a,o)) = \delta(\hat\delta(h), a, o)$. The set $Q$ is not observable, and the
agent sees only the history $h$. Definition~\ref{def:rdp} carries no separate reward
function, since a reward taking finitely many values is subsumed in $O$ by setting
$O = O_{\mathrm{env}} \times R$ and reading $e(q,a)$ as the joint emission over the observed
pair. Appendix~\ref{sec:app-scope} gives that correspondence in full.

\begin{definition}[History likelihood]
\label{def:histlik}
$\displaystyle L_M(h) \;=\; \prod_{t=1}^{|h|} e\bigl(\hat\delta(h_{<t}), a_t\bigr)(o_t)$,
the probability that $M$ emits the observations in $h$ given the actions in $h$.
\end{definition}

The likelihood $L_M$ conditions on the actions and excludes the policy's own probabilities.
This factorization carries the develorequires a logarithm, so expected information gain alonepment that follows, and it rests on an ordering
assumption that \S\ref{sec:measurezero} depends upon, namely that the policy chooses $a_t$
before observing $o_t$.

\begin{definition}[Policy and induced path measure]
A behaviour policy is $\pi : \Hist \to \Delta(A)$. Its history probability is
$\Pi(h) = \prod_{t=1}^{|h|} \pi(h_{<t})(a_t)$, and the induced measure on histories is
$\Prob^\pi_M(h) = \Pi(h)\, L_M(h)$.
\end{definition}

\subsection{Two notions of observational equivalence}

\begin{definition}[Global observational equivalence]
\label{def:globaleq}
$M \sim M'$ iff $L_M(h) = L_{M'}(h)$ for every $h \in \Hist$.
\end{definition}

\begin{definition}[$\pi$-observational equivalence]
\label{def:pieq}
$M \eqpi M'$ iff $L_M(h) = L_{M'}(h)$ for every $h$ with $\Pi(h) > 0$, that is, for every history the behaviour policy can generate.
\end{definition}

Both relations are equivalence relations, $M \sim M'$ implies $M \eqpi M'$ for every $\pi$,
and either relation implies $\Prob^\pi_M = \Prob^\pi_{M'}$. We formalise these facts in Lean
together with the reflexivity, symmetry, and transitivity lemmas that each relation requires.
Throughout, the unqualified term \emph{observational equivalence} denotes the
policy-relative notion of Definition~\ref{def:pieq}, and the relation of
Definition~\ref{def:globaleq} is always marked \emph{global}. Global equivalence is strictly
stronger, and \S\ref{sec:collapse} exhibits a pair that is $\pi$-equivalent while globally
inequivalent.

\section{Policy-Relative Identifiability}
\label{sec:frozen}

Let a prior assign weights $w, w' > 0$ to $M, M'$ respectively. After observing history $h$,
the unnormalized posterior weights are $w\,L_M(h)$ and $w'\,L_{M'}(h)$.

\begin{theorem}[Frozen posterior odds]
\label{thm:frozen}
Let $M \eqpi M'$ and let $h$ be any history with $\Pi(h) > 0$ and $L_{M'}(h) \neq 0$. Then
\[
\frac{w\,L_M(h)}{w'\,L_{M'}(h)} \;=\; \frac{w}{w'}.
\]
Consequently the posterior odds between $M$ and $M'$ equal the prior odds at every sample
size, and no quantity of data collected under $\pi$ moves them.
\end{theorem}

\begin{proof}
$\pi$-equivalence gives $L_M(h) = L_{M'}(h)$, and the common nonzero factor cancels.
\end{proof}

Should the favoured candidate be $\pi$-equivalent to the truth, no data gathered under that behaviour
policy can correct the fitted model. Correcting it becomes a problem of experimental design, which \S\ref{sec:eig} takes up. The quantity our tables report is the complement of $\pi$-equivalence, namely the probability that the policy places on histories separating the pair.

\begin{definition}[Separating mass]
\label{def:sepmass}
Write $D_T = \{\, h \in \Hist_T : L_M(h) \neq L_{M'}(h) \,\}$ for the histories of length $T$
that the two models score differently. The \emph{separating mass} of $\pi$ at horizon $T$ is
the prior-predictive probability of that set, the prior weights being normalized so that
$w + w' = 1$:
\[
m_T(\pi) \;=\; \sum_{h \in D_T} \Pi(h)\,\bigl(w\,L_M(h) + w'\,L_{M'}(h)\bigr).
\]
\end{definition}

Every entry labelled a mass is $m_T(\pi)$, computed in exact rational arithmetic at
$w = w' = 1/2$. It vanishes at every horizon exactly when $M \eqpi M'$, since no history in
$D_T$ has both likelihoods zero, and hence exactly when the criterion of
Theorem~\ref{thm:measurezero}.

\subsection{A two-action T-maze}
\label{sec:collapse}

A natural guess is that $\pi$-observational equivalence reduces to state coverage, so that a
policy visiting every automaton state of the true model identifies it. The following
construction refutes that guess. It needs two actions, and no one-action variant can exhibit
the phenomenon (Remark~\ref{rem:twoactions}).

\begin{itemize}
\item $\mathrm{wait}$ is uninformative, and both models emit a fair coin.
\item $\mathrm{run}$ runs the corridor through to its far end, whereupon the cue model
      reveals the remembered cue deterministically while the blind model emits a fair coin.
\end{itemize}

The cue model $M_{\mathrm{cue}}$ has three automaton states and remembers the first
observation, whereas the blind model $M_{\mathrm{blind}}$ has a single state $q_{\mathrm{b}}$
and ignores history. The behaviour policy waits at every step.

\begin{proposition}[Collapse test]
\label{prop:collapse}
Under the always-$\mathrm{wait}$ policy, $M_{\mathrm{cue}} \eqpi M_{\mathrm{blind}}$, yet
$M_{\mathrm{cue}} \not\sim M_{\mathrm{blind}}$, and the policy reaches each of the two
cue-tracking automaton states with probability $1/2$. The separating history is
$(\mathrm{wait}, \mathrm{left})\,(\mathrm{run}, \mathrm{right})$, on which the cue model
assigns likelihood $0$ and the blind model $1/4$.
\end{proposition}

We formalise this construction in Lean, proving $\pi$-equivalence under the wait policy,
global inequivalence, the fact that the policy reaches both automaton states, and
Theorem~\ref{thm:frozen} applied to this pair. The construction exhibits full state coverage
together with complete failure of identification.

\section{Deciding \texorpdfstring{$\pi$}{pi}-Equivalence}
\label{sec:measurezero}

This section identifies the criterion that Proposition~\ref{prop:collapse} separates from
state coverage. The natural first conjecture, stated over the set of histories on which the
likelihoods differ, is true, but that set need not be regular, so the conjecture yields no
decision procedure (Appendix~\ref{sec:app-conjecture}).

\subsection{Product automaton}

Fix $M = (Q,q_0,\delta,e)$ and $M' = (Q',q_0',\delta',e')$ and form the \emph{product
automaton} $M \times M'$ on state set $Q \times Q'$, with initial state $(q_0, q_0')$ and
transition $\bigl((q,q'),a,o\bigr) \mapsto \bigl(\delta(q,a,o),\,\delta'(q',a,o)\bigr)$. Each
history $h$ drives it to a unique pair $\hat\Delta(h) \in Q \times Q'$.

\begin{definition}[Action support]
\label{def:support}
The \emph{action support} of the behaviour policy is
$A_\pi = \{\, a \in A : \pi(h)(a) > 0 \text{ for some } h \text{ with } \Pi(h) > 0 \,\}$,
the set of actions $\pi$ can ever take.
\end{definition}

\begin{assumption}[History-independent action support]
\label{asm:support}
The action support does not vary with history, so $\pi(h)(a) > 0$ for every
$a \in A_\pi$ and every $h$ with $\Pi(h) > 0$.
\end{assumption}

Assumption~\ref{asm:support} is not required by Theorem~\ref{thm:measurezero}, which weights
every candidate blame step by the probability the policy assigns to it. It is
required by Corollary~\ref{cor:decide} and by Algorithm~\ref{alg:decide}, which test a state
pair without such a weight. Every policy in our suite satisfies it, and
Remark~\ref{rem:policycaveat} gives the reduction for a policy whose support does vary.

\begin{definition}[Disagreeing pairs and the blame language]
\label{def:blame}
The set of \emph{disagreeing pairs}, indexed by the action that witnesses the
disagreement, is
\[
\begin{aligned}
B_\pi \;=\; \bigl\{\, ((q,q'),a) &\in (Q \times Q') \times A_\pi \;:\; \exists\, o \in O \\[-1pt]
  &\text{with } e(q,a)(o) \neq e'(q',a)(o) \,\bigr\},
\end{aligned}
\]
and the \emph{blame language} is
\[
L_{B_\pi} \;=\; \bigl\{\, h\,(a,o) \in \Hist \;:\; \bigl(\hat\Delta(h),\,a\bigr) \in B_\pi \,\bigr\},
\]
the histories drive the product automaton into a disagreeing pair and take the action on which it disagrees.
\end{definition}

Restricting the quantifier to $A_\pi$ makes the pair test policy-relative, and carrying the
trailing action weights each candidate blame step by $\pi(h)(a)$.
Appendix~\ref{sec:app-scope} shows why each device is needed. The set $B_\pi$ depends only
on the two emission tables and the action support, so deciding membership is a finite check
that involves no products of probabilities, and $L_{B_\pi}$ is recognized by $M \times M'$
with an accepting condition on state pair and action. The blame
language is therefore regular by construction, and it is the object over which the question
should be posed.

\paragraph{Relation to classical equivalence checking.} The product is the construction that
decides deterministic finite automaton (DFA) equivalence \citep{hopcroftkarp1971} and probabilistic bisimulation
\citep{larsen1991,givan2003}, but the question asked of it differs in the two devices
above. Those notions quantify over every input, so each refines global equivalence
(Definition~\ref{def:globaleq}) and each calls the pair of Proposition~\ref{prop:collapse}
distinguishable. That verdict is correct about the two automata, but the corpus exhibits no
such distinction at any sample size. The trailing action and the action-before-observation
ordering also keep the blame-step weight $\Pi(h)\,\pi(h)(a)$ common to both models, so no
policy can obtain a certificate by weighting the disagreeing step unevenly under the two
candidates (Appendix~\ref{sec:app-scope}). Bisimulation metrics \citep{ferns2004} measure
how far apart two states of one model are and do not condition on a policy. The quantitative
refinement here is instead the separation $\gamma$ of \S\ref{sec:scale}.

\subsection{The characterization at finite horizon}

Both models assign $L_{B_\pi}$ the same probability under $\pi$
(Appendix~\ref{sec:app-scope}), so we write $\Prob^\pi(L_{B_\pi})$ without a model subscript.

\begin{theorem}[Measure-zero characterization]
\label{thm:measurezero}
$M \eqpi M'$ if and only if $\Prob^\pi(L_{B_\pi}) = 0$, that is, if and only if the behaviour
policy assigns probability zero to ever driving the product automaton into a disagreeing pair
\emph{and there taking an action on which the two emission rows differ}. The statement holds
at every finite horizon $T$, for histories of length at most $T$, and at infinite horizon, for
the event that some prefix of the trajectory lies in $L_{B_\pi}$.
The behaviour policy is arbitrary, so its action support may vary with history and
Assumption~\ref{asm:support} is not used.
\end{theorem}

\begin{proof}[Proof sketch, finite horizon]
Both directions turn on one decomposition. Call $h\,(a,o) \in L_{B_\pi}$ a \emph{first entry}
when no proper prefix of it lies in $L_{B_\pi}$. When $\Pi(h) > 0$ every action taken in $h$
lies in $A_\pi$, so at every step of the prefix of a first entry the two models
consult emission rows that agree and $L_M(h) = L_{M'}(h)$. The first entries partition
$L_{B_\pi}$ into blocks $\{\,h\,(a,o') : o' \in O\,\}$, and because the rows are normalized
and the policy factor $\Pi(h)\,\pi(h)(a)$ does not depend on $o'$, each block has mass
$\Pi(h)\,\pi(h)(a)\,L_M(h)$ under either model.
Appendix~\ref{sec:app-proof} derives both directions from it.
\end{proof}

The theorem requires normalized emission rows and the action-before-observation ordering,
both of which belong to Definition~\ref{def:rdp} and to the factorization above, and
Remark~\ref{rem:hyp} identifies the block decomposition.

\subsection{A linear-time decision procedure}
\label{sec:decide}

\begin{corollary}[Decision procedure]
\label{cor:decide}
Under Assumption~\ref{asm:support}, $\pi$-equivalence is decided by reachability in the
product automaton. Explore from $(q_0,q_0')$ along transitions the policy gives positive
probability, and at each pair reached test the emission rows of the actions in $A_\pi$ only,
reporting equivalence if and only if no pair in $B_\pi$ is reached. This runs in time
$O(|Q|\,|Q'|\,|A|\,|O|)$, linear in the size of the product automaton, and is independent of
the horizon.
\end{corollary}

Assumption~\ref{asm:support} makes the test at a pair independent of the history that reached
it, so one visit settles each pair and the bound does not depend on the horizon. \textsc{Pec} returns a certificate of
$\pi$-equivalence when no disagreeing pair is reachable and a witness naming the missing
experiment when one is. Algorithm~\ref{alg:decide} in Appendix~\ref{sec:supp-algo} states it,
and the frontier is a queue, so the witness returned has minimum depth. \textsc{Pec} reports
that depth as the \emph{separating horizon}, the length of the shortest separating history,
which counts the witness action itself. \textsc{Pec} thereby replaces enumeration over
histories, which costs time exponential in the horizon.

\paragraph{The witness.} When the models are not $\pi$-equivalent, the search returns the
disagreeing pair it reached together with the action on which the emission rows differ. On
the T-maze under a deviation this witness is
$\bigl(\,(q_{\mathrm{L}}, q_{\mathrm{b}}),\ \mathrm{run}\,\bigr)$, naming $\mathrm{run}$ in
the state pair where the cue model remembers a left cue and the blind model sits in its
only state. That is the missing experiment, named without any search over policies, and it is
the input to the design problem of \S\ref{sec:eig}.

\paragraph{Validation.} On 3700 model pairs, 2500 random and 1200 constructed to be
$\pi$-equivalent, \textsc{Pec} agreed with a brute-force check that compares the two likelihoods on
every policy-supported history up to a fixed horizon.

\subsection{Behaviour at scale}
\label{sec:scale}

\begin{figure}[t]
\centering
\includegraphics[width=\columnwidth]{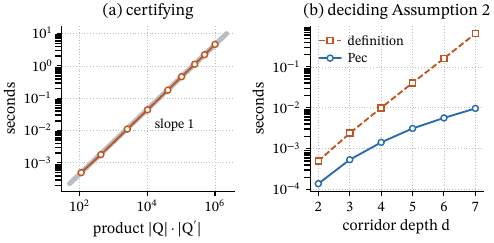}
\caption{Cost of deciding identifiability, on log axes. (a)~certifying $\pi$-equivalence
against product size. (b)~cost of deciding Assumption~2 of \citet{regorl} from its definition
versus with \textsc{Pec}.}
\label{fig:certify}
\end{figure}

The bound of Corollary~\ref{cor:decide} admits direct measurement. On constructed
$\pi$-equivalent pairs built from two cycles of coprime length, whose reachable product is
exactly $|Q|\cdot|Q'|$, \textsc{Pec} explores from $110$ to $1{,}001{,}000$ state pairs at
a cost flat to within $8\%$ per pair and a least-squares slope of $\log(\text{time})$ against
$\log(\text{pairs})$ of $1.003$, plotted in Figure~\ref{fig:certify}(a). A certificate costs
the whole reachable product while a
witness costs only the part preceding it, so at a product of $1{,}001{,}000$ a witness at
depth $999$ returns in $7.2\times10^{-3}$ seconds against $4.65$ seconds to certify the same
instance. Appendix~\ref{sec:app-scaling} reports the full sweeps together with the
action-support factor.

Definition~\ref{def:blame} tests emission rows for exact inequality, whereas estimated
tables are built from counts. Two rows estimated from the same distribution then differ with
probability approaching one as the corpus grows, so the exact test almost never certifies.
Replacing the test by $\max_o |e(q,a)(o) - e'(q',a)(o)| > \tau$ for a tolerance $\tau$
repairs this. Write $\eta$ for the sup-norm error of the estimated tables over the rows
\textsc{Pec} consults and $\gamma$ for the smallest sup-norm separation among pairs whose
exact rows disagree. Truly equal rows are estimated at most $2\eta$ apart and truly separated
ones at least $\gamma - 2\eta$ apart, so the tolerant test returns the exact-arithmetic
verdict whenever $2\eta < \tau < \gamma - 2\eta$, a window non-empty as soon as
$\eta < \gamma/4$. Since $\eta$ shrinks as the corpus grows while $\gamma$ is a property of
the candidates, the window opens at a finite sample size. On the corridor family at 600
trajectories per candidate, $\eta$ ranges from $0.072$ to $0.098$ against $\gamma = 1/2$, and
every tolerance inside the window reproduces exact arithmetic (Appendix~\ref{sec:app-robust}).

\section{Experiments}
\label{sec:experiments}

\paragraph{Code and proofs.} The implementation, the synthetic environments, and the Lean~4
development containing the machine-checked proofs are released at
\url{https://github.com/Kcbir/pec}.

\paragraph{Protocol.} All environments are synthetic Regular Decision Processes with known
ground truth, and all data are simulated. Rotating MAB and Cheat MAB, two multi-armed bandit (MAB) variants, are taken from
\citet{abadi2020}, and Rotating Maze is a reduced version of the grid-world domain of the same
name introduced there, so three of the four come from established benchmarks for this model
class. Known ground truth makes the identifiability claims checkable, since a frozen
posterior can be verified only against the model the data failed to distinguish. Inference
uses exact rational arithmetic throughout, with no sampling and no training, emission
parameters being integrated out under a Dirichlet prior to keep the structure posterior
rational (Appendix~\ref{sec:app-ml}). Entropy requires a logarithm, so expected information gain alone is computed in floating point, whereas the separating mass of Definition~\ref{def:sepmass} is exact.

\paragraph{Scope.} A standardized partially observable benchmark such as POPGym
\citep{popgym} supplies neither the candidate pair nor the behaviour policy that \textsc{Pec}
takes as input. The method requires a known pair whatever the size of the environment, and
\S\ref{sec:scale} addresses scale directly.

\subsection{The distinguishability assumption of RegORL}
\label{sec:regorl}

The guarantee of \citet{regorl} rests on their Assumption~2, that the behaviour policy has
$L^p_\infty$-distinguishability at least $\mu_0 > 0$ between every pair of distinct
same-layer states of the minimal RDP. Theorem~6 of \citet{regorl} requires a corpus whose
size is proportional to $1/\mu_0$ (Appendix~\ref{sec:app-regorl}), so the requirement
diverges as $\mu_0$ goes to zero. Table~\ref{tab:ident} computes that parameter
in exact rational arithmetic beside the separating mass of Definition~\ref{def:sepmass}.

\begin{table*}[t]
\centering\small
\caption{Separating mass and the distinguishability parameter of \citet{regorl} on four
environments, as exact rationals, against the verdict \textsc{Pec} returns on the same
instance.}
\label{tab:ident}
\begin{tabular}{lrccccll}
\toprule
& & \multicolumn{2}{c}{separating mass} & \multicolumn{2}{c}{$\mu_0$} & & \\
\cmidrule(lr){3-4}\cmidrule(lr){5-6}
Environment & $T$ & under $\pi$ & under $\pi'$ & under $\pi$ & under $\pi'$
  & Assumption~2 & \textsc{Pec} \\
\midrule
T-maze        & 2 & $\mathbf{0}$ & $1/2$     & $\mathbf{0}$ & $1/2$   & fails    & equivalent \\
Rotating MAB  & 3 & $\mathbf{0}$ & $245/512$ & $\mathbf{0}$ & $1/4$   & fails    & equivalent \\
Cheat MAB     & 3 & $\mathbf{0}$ & $1/9$     & $\mathbf{0}$ & $1/24$  & fails    & equivalent \\
Rotating Maze & 4 & $\mathbf{0}$ & $1/8$     & --           & $9/256$ & holds (trivially) & equivalent \\
\bottomrule
\end{tabular}
\end{table*}

On three of the four environments $\mu_0 = 0$ as an exact rational, so Assumption~2 fails and
the required corpus size diverges. The separating mass is an exact rational zero on each of
them, and two of the three are benchmarks from \citet{abadi2020}. On Rotating Maze the
assumption instead holds trivially, since the behaviour policy reaches a single state per
layer and there is no pair of distinct states to compare. The two criteria overlap without
either containing the other, because $\mu_0$ compares two states of the true model that the
policy reaches while $\pi$-equivalence compares two candidate models. An ambiguity in a region the policy never enters preserves Assumption~2, as in Rotating Maze, while \textsc{Pec} certifies $\pi$-equivalence.

\paragraph{Deciding Assumption~2.} For a pair of states distinct in the minimal RDP,
$L^p_\infty > 0$ holds exactly when some $\pi$-supported trace prefix within the remaining
horizon receives different probability from the two states. By
Theorem~\ref{thm:measurezero} applied to the sub-automaton started at $(q,q')$, this is the
condition that \textsc{Pec}, started from $(q,q')$, reports a separating horizon no greater
than the number of steps remaining. The definition maximizes over $(|A|\,|O|)^T$ trace
prefixes at horizon $T$, whereas \textsc{Pec} explores at most $|Q|\cdot|Q'|$ pairs without
reference to the horizon. On the
corridor family of \S\ref{sec:tier1c} the measured ratio grows from $3.7$ at depth $2$ to
$67.5$ at depth $7$ (Figure~\ref{fig:certify}(b)), and the two methods agree everywhere
(Appendix~\ref{sec:app-regorl}). Under the designed deviation $\mu_0$ becomes positive in
every environment, so the experiment \textsc{Pec} names is the one that restores their
assumption.

\subsection{Selecting policy deviations}
\label{sec:eig}

Theorem~\ref{thm:frozen} shows that a $\pi$-equivalent pair is uncorrectable under $\pi$, so
the remedy is to change $\pi$. We select policy deviations by expected information gain (EIG) under
the posterior over the candidate structures, which applies Bayesian active learning by disagreement
(BALD) \citep{bald} to automaton
structure in a controlled process. Let $\{M_i\}$ be the candidate models with prior weights
$\{w_i\}$, and let a candidate experiment be an \emph{open-loop action sequence}
$\alpha = (a_1,\dots,a_d)$, executed by the degenerate policy $\pi_\alpha$ that plays $a_t$
at step $t$ regardless of what it observes. The prior-predictive probability of a history
$h \in \Hist_d$ under $\pi_\alpha$ is
\[
P_\alpha(h) \;=\; \Pi_\alpha(h) \sum_i w_i\, L_{M_i}(h),
\]
where $\Pi_\alpha(h) = \prod_{t=1}^{d} \mathbf{1}[a^h_t = a_t]$ takes the value $1$ on the
histories whose actions are $\alpha$, and the posterior after $h$ is
$w_i(h) \propto w_i L_{M_i}(h)$. The expected information gain of
the sequence is
\[
\mathrm{EIG}(\alpha) \;=\; H\bigl(\{w_i\}\bigr)
\;-\;\!\!\sum_{\substack{h \in \Hist_d \\ P_\alpha(h) > 0}}\!\! P_\alpha(h)\;
H\bigl(\{w_i(h)\}\bigr),
\]
with $H$ the Shannon entropy in nats. Greedy design maximizes $\mathrm{EIG}$ over the
one-step extensions of the sequence built so far, breaking ties by the order in which actions
are enumerated, while exhaustive design at lookahead $\ell$ maximizes it over all $|A|^\ell$
sequences at a cost of $O((|A|\,|O|)^\ell)$ likelihood computations. \textsc{Pec} returns the
required $\ell$ in linear time, so the depth of that exponential search is known before it
is run.

\begin{remark}[Zero first-step gain]
\label{prop:nogradient}
In every environment in our suite, the expected information gain of the first action is
\emph{exactly} zero.
\end{remark}

In each environment the candidates' emission rows at their initial states agree on every
action, so the separating horizon is at least $2$, whereas a pair that disagrees at its
initial states would have positive first-step gain. A greedy one-step rule sees
equal gain for every first action here, and its tie-break determines what it does. We compare greedy against exhaustive design over sequences whose length equals
the separating horizon, running greedy twice with the action order reversed.

\begin{table}[t]
\centering\small
\caption{Separating mass achieved by each design rule, as exact rationals. The two greedy
columns differ only in the order in which actions are enumerated for the tie-break, and $\ell$
is the lookahead at which exhaustive design is run.}
\label{tab:greedy}
\begin{tabular}{lcccc}
\toprule
& & \multicolumn{2}{c}{greedy} & \\
\cmidrule(lr){3-4}
Environment & $\ell$ & order & reversed & exhaustive \\
\midrule
T-maze        & 2 & $1$   & $1$   & $1$ \\
Rotating MAB  & 2 & $1/2$ & $1/4$ & $3/4$ \\
Cheat MAB     & 3 & $\mathbf{0}$ & $\mathbf{0}$ & $1$ \\
Rotating Maze & 4 & $\mathbf{0}$ & $1$   & $1$ \\
\bottomrule
\end{tabular}
\end{table}

Exhaustive design succeeds in every environment, whereas greedy design succeeds on two of the
four. Cheat MAB is a structural failure of the greedy rule, since the experiment requires two
different actions in sequence and neither action yields any gain alone, so greedy achieves
separating mass $0$ under both tie-breaks while exhaustive design recovers the unlock
sequence $(\mathrm{arm1}, \mathrm{arm2})$ from the posterior without being given it. On
Rotating Maze the same greedy rule achieves $0$ or $1$ according to the order in which
actions are enumerated, so its success there derives from the tie-break alone, and myopic
design is therefore unreliable.

\subsection{A parameterized family and the required lookahead}
\label{sec:tier1c}

The four environments above are fixed, small instances, which leaves open whether the phenomenon
survives a controlled change in difficulty. A member of depth $d$ is a corridor of $d$ rooms
with the missing experiment placed at the far end. The behaviour policy walks the corridor
without probing, and the deviation walks and then probes. The alphabet, the emission values,
and the way the two candidates differ stay fixed, and only the corridor length varies, so a curve
over the family isolates depth without confounding it with model size or stochasticity.
For every member, \textsc{Pec} reports no reachable disagreeing pair under $\pi$ and reports
a separating horizon of exactly $d$ under the deviation (Table~\ref{tab:tier1c},
Appendix~\ref{sec:app-corridor}), so the parameter is read off the structure, with no
sweep over horizons.

Sweeping the design lookahead $\ell$ against depth $d$ gives a transition that is exact,
since a lookahead below the depth finds nothing while a lookahead equal to it finds the
experiment (Table~\ref{tab:lookahead}, Appendix~\ref{sec:app-corridor}). For open-loop
search this follows from the definitions, and the sweep confirms it in the implementation, so
one structural quantity serves two roles. The separating horizon measures data
difficulty, since it ranks the suite in the small-data crossover of
Appendix~\ref{sec:ablations}, and it measures planning difficulty, since it equals the
lookahead a designer needs.

\subsection{End-to-end evaluation}
\label{sec:method}

On the four environments of Table~\ref{tab:ident} the designed experiment does not improve on
uniform exploration. Over ten seeds, a corpus gathered under $\pi$ leaves the $\log_{10}$
Bayes factor at exactly zero in all four environments and on every seed. On three of them
the deviation assigns equal probability to every action, so the designed and the random
corpus coincide. On Rotating MAB they differ by less than the seed variation, with
means and half-ranges of $161.93 \pm 20.04$ and $148.86 \pm 17.89$
(Appendix~\ref{sec:app-pipeline}). The separating horizons here run from $2$ to $4$ over $2$
or $3$ actions, so uniform exploration finds the missing experiment without help.
\label{sec:method-negative}

The boundary between the two regimes follows from the structure of the family. To find the
missing experiment of a member with separating horizon $d$, a uniform policy must
emit a specific action sequence of length $d$, namely $d-1$ steps along the corridor and then
the probe, which it does with probability near $2^{-d}$, whereas a designed experiment
executes that sequence every time. The advantage therefore grows exponentially in $d$, which \S\ref{sec:tier1c}
identifies as the difficulty axis.
\label{sec:method-depth}

\begin{table}[h]
\centering\small
\setlength{\tabcolsep}{4pt}
\caption{Mean $\log_{10}$ Bayes factor over 10 seeds by corridor depth, at budgets of 5 and
25 trajectories.}
\label{tab:depth}
\begin{tabular}{rrrcrrc}
\toprule
& \multicolumn{3}{c}{budget 5} & \multicolumn{3}{c}{budget 25} \\
\cmidrule(lr){2-4}\cmidrule(lr){5-7}
depth $d$ & rand. & des. & ratio & rand. & des. & ratio \\
\midrule
2 & $1.51$ & $3.01$ & $2.0\times$ & $8.49$ & $15.05$ & $1.8\times$ \\
4 & $0.57$ & $3.01$ & $5.3\times$ & $3.43$ & $15.05$ & $4.4\times$ \\
6 & $0.39$ & $3.01$ & $7.7\times$ & $1.51$ & $15.05$ & $10.0\times$ \\
8 & $0.06$ & $3.01$ & $\mathbf{50.0\times}$ & $0.39$ & $15.05$ & $\mathbf{38.5\times}$ \\
\bottomrule
\end{tabular}
\end{table}

The designed column does not vary with depth, because the designed experiment reaches the
separating history every time, whereas the yield of uniform exploration decays exponentially.
Appendix~\ref{sec:app-eqtest} repeats the comparison against the access-sequence heuristic
that active automata learning libraries ship, which converts the exponential decay into a
polynomial one and, at depth eight, narrows the gap in separating mass from a factor of $256$
to a factor of about $8$. Which regime a problem occupies is settled by the linear-time detection
step, which returns $d$ together with the witness.

\section{Limitations}
\label{sec:limits}

The observation set is finite throughout, and continuous observations lie outside every
result here. The decision procedure is exact, so on estimated tables it requires the tolerance
of \S\ref{sec:scale} and holds only inside the window $2\eta < \tau < \gamma - 2\eta$,
which needs ever larger corpora as the candidates approach equivalence. Exhaustive design
remains exponential in the lookahead, and \textsc{Pec} supplies the required depth without
reducing that cost. Exact enumeration bounds the hypothesis space at five automaton states. When $\pi$ is unknown, \textsc{Pec} reads its action support off the
corpus, so an action taken too rarely to appear there is treated as never taken.
\textsc{Pec} also decides a question about a given pair, and whether a corpus identifies a
model within an implicitly specified hypothesis class remains open.

\section{Conclusion}

This paper has addressed whether data collected under a fixed behaviour policy can distinguish
two candidate Regular Decision Processes, a question on which existing offline guarantees for
this model class depend but which they provide no means to verify. We have shown that
$\pi$-observationally equivalent candidates leave the posterior odds at the prior odds at
every sample size, and that this condition is not implied by state coverage. We have
characterized $\pi$-equivalence exactly over the disagreeing state--action pairs of the
product automaton, whose blame language, unlike the set of histories on which the likelihoods
differ, is regular by construction. The characterization yields \textsc{Pec}, which decides
equivalence in time linear in the size of the product automaton and, when equivalence fails,
names the experiment that the corpus lacks. Applied to the distinguishability assumption of
\citet{regorl}, the procedure establishes that the assumption fails on three of our four test
environments and that the experiment it names restores the assumption in each case. Extending the analysis to continuous observations and implicitly specified hypothesis classes remains for future work.

\bibliographystyle{abbrvnat}
\bibliography{references}

\clearpage

\appendix
\thispagestyle{empty}
\onecolumn
\aistatstitle{Exact Distinguishability in Non-Markovian Decision Processes: Supplementary Materials}

\renewcommand{\thesection}{S\arabic{section}}
\renewcommand{\thefigure}{S\arabic{figure}}
\renewcommand{\thetable}{S\arabic{table}}
\renewcommand{\thealgorithm}{S\arabic{algorithm}}
\setcounter{figure}{0}
\setcounter{table}{0}

\section{The decision procedure}
\label{sec:supp-algo}

Algorithm~\ref{alg:decide} states \textsc{Pec}, the procedure of \S\ref{sec:decide}. It explores the
product automaton within the support of the behaviour policy and halts at the first
disagreeing pair. A returned witness names the missing experiment, and a returned $\bot$
certifies $\pi$-equivalence at every horizon.

\begin{algorithm}[!ht]
\caption{\textsc{Pec}: decide $\pi$-equivalence and return a witness.}\label{alg:decide}
\begin{algorithmic}[1]
\Require models $M = (Q,q_0,\delta,e)$ and $M' = (Q',q_0',\delta',e')$, behaviour policy $\pi$
\Ensure a pair $((q,q'),a)$ naming the missing experiment, or $\bot$ if $M \eqpi M'$
\State $A_\pi \gets \{\, a \in A : \pi(\lambda)(a) > 0 \,\}$
  \Comment{action support, history-independent by assumption}
\State $s_0 \gets (q_0, q_0')$;\quad $S \gets \{s_0\}$;\quad
  $F \gets$ queue holding $s_0$
\While{$F \neq \emptyset$}
  \State dequeue $(q,q')$ from $F$
  \ForAll{$a \in A_\pi$}
    \If{$\exists\, o \in O$ with $e(q,a)(o) \neq e'(q',a)(o)$}
      \State \Return $((q,q'),a)$ \Comment{reachable disagreeing pair}
    \EndIf
  \EndFor
  \ForAll{$a \in A_\pi$ and $o \in O$ with $e(q,a)(o) > 0$}
    \State $s \gets \bigl(\delta(q,a,o),\, \delta'(q',a,o)\bigr)$
    \If{$s \notin S$}
      \State $S \gets S \cup \{s\}$;\quad enqueue $s$ on $F$
    \EndIf
  \EndFor
\EndWhile
\State \Return $\bot$ \Comment{no disagreeing pair reachable}
\end{algorithmic}
\end{algorithm}

Each state pair enters $S$ at most once, and the work at a state is
$O(|A|\,|O|)$, so \textsc{Pec} runs in time $O(|Q|\,|Q'|\,|A|\,|O|)$ and uses
$O(|Q|\,|Q'|)$ space. Recording the level at which the
loop returns gives the minimum separating horizon, which is the depth reported in
Table~\ref{tab:tier1c} and the lookahead required in Table~\ref{tab:lookahead}.

Two details of the frontier determine which experiment is returned and what it costs. The
disagreement test runs as each pair is dequeued, before the reachable set is complete, so a
witness costs only the part of the product that precedes it (Appendix~\ref{sec:app-scaling}
measures this), while a certificate costs the whole of it. Because $F$ is a queue, pairs are
visited in order of distance from $(q_0,q_0')$, so the witness returned has \emph{minimum}
depth and is the same on every run. An implementation that instead scans the reachable set as
an unordered collection returns an equally valid witness that depends on the iteration order
of that collection, and on the T-maze it names $q_{\mathrm{L}}$ or $q_{\mathrm{R}}$
accordingly.

\textsc{Pec} probes the action support of $\pi$ at the empty history, which is exact under
Assumption~\ref{asm:support} (Remark~\ref{rem:policycaveat}). A policy whose support varies
with history requires folding the support into the product automaton, which enlarges the
automaton without changing the argument. This extension limits the scope of the procedure as
well as its cost. Folding is available only when the support is a \emph{regular} function of
history, that is, when some finite automaton with state set $Q_\pi$ reads the
action--observation sequence and determines $A_\pi$ from its state. The search then runs over
$Q \times Q' \times Q_\pi$, and the two bounds above pick up a factor of $|Q_\pi|$, giving
$O(|Q|\,|Q'|\,|Q_\pi|\,|A|\,|O|)$ time and $O(|Q|\,|Q'|\,|Q_\pi|)$ space, still with no
horizon. If
the support is an arbitrary function of history the folded automaton is infinite, reachability
over it is not a finite computation, and \textsc{Pec} does not apply. Regularity of the
policy's support is therefore the outer limit of the method.

\section{The T-maze construction}
\label{sec:supp}

Figure~\ref{fig:mechanism} displays the two hypotheses of \S\ref{sec:collapse}. The cue model
$M_{\mathrm{cue}}$ records the first observation and never revises it, whereas the blind model
$M_{\mathrm{blind}}$ has a single state. They differ only in what the action $\mathrm{run}$
emits, and $\mathrm{run}$ is the action the behaviour policy never takes. Over histories of
length two, the support of the behaviour policy and the set on which the two likelihoods
disagree are therefore disjoint, which is the $\pi$-equivalence asserted in
Proposition~\ref{prop:collapse}. Each of the two cue-tracking states is visited with
probability exactly $1/2$ under the behaviour policy, which is the coverage fact recorded in
the same proposition.

\begin{figure}[!ht]
\centering
\includegraphics[width=0.5\columnwidth]{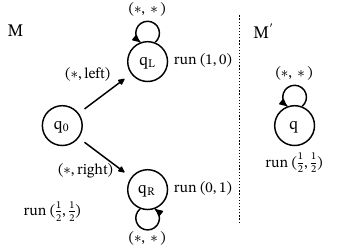}
\caption{The two hypotheses of the T-maze construction. Edges carry the
action--observation pair that drives them, as in Definition~\ref{def:rdp}, with $\ast$
standing for any element of the alphabet in that slot. Neither transition function consults
the action, so every edge is labelled $\ast$ there, and the $(\ast,\ast)$ self-loops on
$q_{\mathrm{L}}$, $q_{\mathrm{R}}$ and $q_{\mathrm{b}}$ are the ones the always-$\mathrm{wait}$
policy traverses.}
\label{fig:mechanism}
\end{figure}

\section{Value of a structural description}
\label{sec:supp-sentence}

A structural description can also enter as a prior. We encode a sentence by hand as a
predicate on transition structures and take the prior $\rho_L$ uniform on the structures it
admits, contaminated towards an uninformative prior $\rho_0$ at rate $\varepsilon$, giving
$(1-\varepsilon)\rho_L + \varepsilon\,\rho_0$ \citep{berger1986eps}. The hypothesis space is
every transition structure on at most three states over the T-maze observations, one per
state-renaming class, $229$ in all, and $\rho_0$ is uniform over it. Three sentences of
increasing specificity admit $179$, $19$, and $1$ of these structures. The \emph{vague}
sentence, ``history matters'', requires the first observation to lead to different states.
The \emph{partial} sentence, ``remembers, commits'', also requires one of those states to be
absorbing, and the \emph{saturated} sentence requires both to be, which pins the truth and so
amounts to a specification of the model. The wrong description of panel~\ref{fig:sent-wrong}
is the negation of the vague one. Panel~\ref{fig:sent-contract} shows posterior contraction towards the true model under priors
of increasing specificity, where the gap opens in the small-data regime and closes as data
accumulate. Panel~\ref{fig:sent-gain} states the same fact as a ratio against the
uninformative prior, for the partial sentence alone, which is the $19/229$ curve of
panel~\ref{fig:sent-contract} divided by the no-sentence baseline. Panel~\ref{fig:sent-wrong} covers misspecification, where a wrong
description at $\varepsilon = 0$ places zero prior mass on the truth and the posterior never
reaches it, while any $\varepsilon > 0$ recovers it.

\begin{figure}[!ht]
\centering
\begin{subfigure}[t]{0.32\textwidth}
  \centering
  \includegraphics[width=\linewidth]{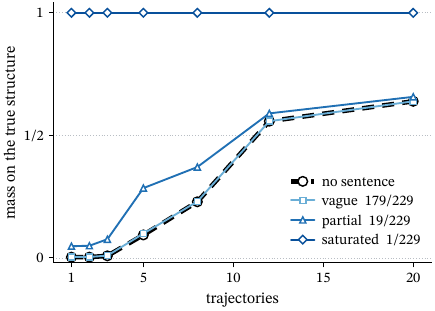}
  \caption{Posterior contraction.}
  \label{fig:sent-contract}
\end{subfigure}\hfill
\begin{subfigure}[t]{0.32\textwidth}
  \centering
  \includegraphics[width=\linewidth]{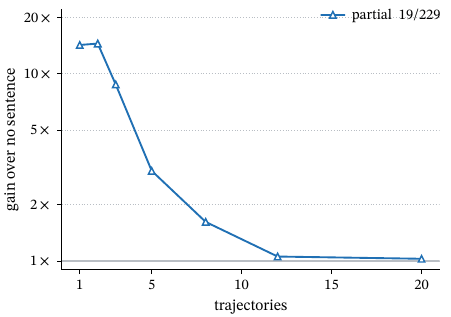}
  \caption{Gain from the partial sentence.}
  \label{fig:sent-gain}
\end{subfigure}\hfill
\begin{subfigure}[t]{0.32\textwidth}
  \centering
  \includegraphics[width=\linewidth]{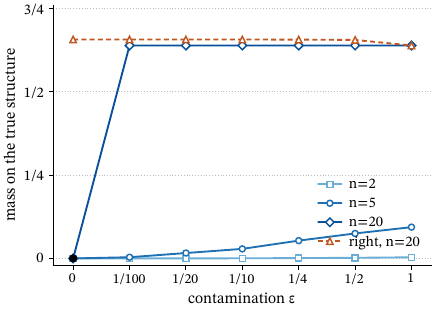}
  \caption{A wrong description.}
  \label{fig:sent-wrong}
\end{subfigure}
\caption{Effect of description specificity and of misspecification on the posterior.
Panel~(a) plots all three sentences against the no-sentence
baseline. Panel~(b) plots the partial sentence ``remembers, commits'' alone, named on the
axes, since the vague sentence is worth $1.3\times$ at its peak and the saturated one is a
specification.}
\label{fig:sentence}
\end{figure}

\section{Seed variation and prior sensitivity}
\label{sec:supp-ablations}

Figure~\ref{fig:robust} plots the seed variation, prior sensitivity and small-data crossover
reported in Appendix~\ref{sec:ablations}. Panel~\ref{fig:rob-seeds} shows the mean and range
of the $\log_{10}$ Bayes factor over 12 seeds, panel~\ref{fig:rob-alpha} the sweep over the
Dirichlet concentration, and panel~\ref{fig:rob-occam} the number of seeds that favour the
memoryless null as the number of trajectories grows.

\begin{figure}[!ht]
\centering
\begin{subfigure}[t]{0.32\textwidth}
  \centering
  \includegraphics[width=\linewidth]{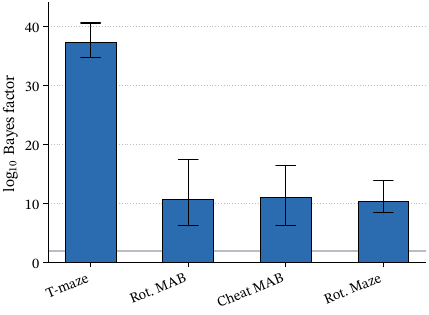}
  \caption{Variation over 12 seeds.}
  \label{fig:rob-seeds}
\end{subfigure}\hfill
\begin{subfigure}[t]{0.32\textwidth}
  \centering
  \includegraphics[width=\linewidth]{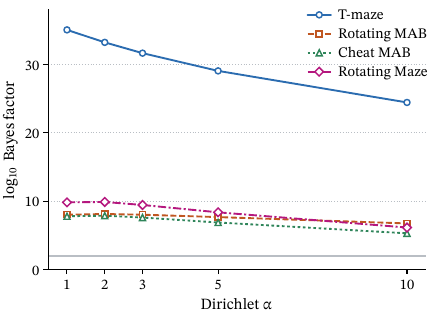}
  \caption{Dirichlet concentration sweep.}
  \label{fig:rob-alpha}
\end{subfigure}\hfill
\begin{subfigure}[t]{0.32\textwidth}
  \centering
  \includegraphics[width=\linewidth]{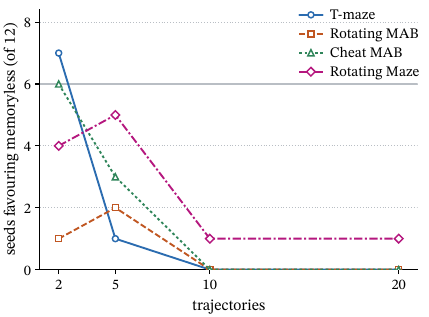}
  \caption{Small-data crossover.}
  \label{fig:rob-occam}
\end{subfigure}
\caption{Seed variation, prior sensitivity, and the small-data crossover.}
\label{fig:robust}
\end{figure}

\section{Extended Related Work}
\label{sec:app-related}

\subsection{Non-Markovian formalisms}
\label{sec:rel-formalisms}

\begin{table}[!ht]
\centering\small
\caption{Location of history-dependence across non-Markovian formalisms. MDP: Markov decision
process; POMDP: partially observable MDP; PSR: predictive state representation.}
\begin{tabular}{lll}
\toprule
Formalism & History-dependence in & Notes \\
\midrule
Reward machines & reward only & dynamics stay Markov; mature literature \\
Regular Decision Processes & dynamics \textbf{and} reward & more general; less studied \\
MDP (incl.\ $L^*_{\mathrm{MDP}}$ setting) & none, state is Markov & structure learning $\neq$ non-Markovian dynamics \\
POMDP & latent state & more general, much less tractable \\
PSR / spectral & predictive state & learning-theoretic, different toolchain \\
\bottomrule
\end{tabular}
\end{table}

\citet{toroicarte2022rm} give the standard reference for reward machines, where recent work
has extended expressiveness through pushdown reward machines \citep{pushdownrm} and through
neural and visual variants. Our question instead concerns identifiability from data.
\citet{brafman2019rdp} introduced Regular Decision Processes, and \citet{abadi2020} gave the
first method for learning and solving them. Later learning results are offline reinforcement
learning with PAC guarantees \citep{regorl,iclr2025rdp,tractableoffline}, alongside Markov
abstractions for PAC reinforcement learning in non-Markov decision processes
\citep{markovabstractions}, online learning of non-Markovian reward models
\citep{onlinenmrm}, learning task automata through hidden Markov models \citep{taskautomata},
and automata cascades \citep{cascades}. None of this work uses priors or side information.
Markov-order testing goes back to \citet{billingsley1961}, with modern order tests based on
conditional mutual information and an exact test using surrogate data
\citep{exactmarkovorder}, none of which handles actions and a behaviour policy.

\subsection{Offline learning in Regular Decision Processes}
\label{sec:rel-offline}

\citet{regorl} give the first provably efficient offline reinforcement learning framework for
Regular Decision Processes. The method learns an automaton-based abstraction from a fixed
corpus and then runs an offline algorithm over the resulting Markov model, and its PAC bound
is stated relative to a distinguishability quantity together with a concentrability condition
on the behaviour policy. \citet{tractableoffline} and \citet{iclr2025rdp} continue that line,
the latter replacing the $L_\infty^p$-distinguishability parameter with a formal-language
metric because the original quantity decays exponentially in the number of automaton states,
which is the reason we discuss that parameter in \S\ref{sec:intro}.

RegORL and this paper divide the labour between them, so a head-to-head comparison would not
be meaningful. RegORL returns a policy together with a suboptimality bound that is conditional
on the corpus being distinguishing. It does not decide whether that condition holds, and the
corpus alone cannot decide it, because the condition concerns models that the corpus fails to
separate, whereas \textsc{Pec} decides it from the candidate models (\S\ref{sec:regorl}).
The two methods therefore
return objects that are not commensurable as estimators, and they are related through the
parameter itself. Their $\mu_0$ is a property of the true RDP and the policy together,
obtained by comparing the suffix distributions of two states the policy reaches, whereas
\textsc{Pec} compares two candidate models. 

Appendix~\ref{sec:app-regorl} computes $\mu_0$ on all four environments. On three of them it
is exactly zero, so Assumption~2 fails and the bound is vacuous, and on each of them
Table~\ref{tab:ident} also reports a rational zero for the separating mass. Two of these three
environments are taken from \citet{abadi2020}. On Rotating Maze, and on the whole corridor
family, the assumption is instead satisfied trivially, since the policy reaches a single
state per layer and there is no pair of distinct states to compare. The parameter does not
flag those instances, while \textsc{Pec} certifies $\pi$-equivalence on every one of them, so
does not vanish on every instance our criterion flags.

\subsection{Bayesian inference over automata and reward machines}
\label{sec:rel-bayesauto}

Bayesian inference over automata exists, but it stands apart from decision processes.
Probabilistic Deterministic Infinite Automata \citep{pfau2010pdia} give a nonparametric
Bayesian prior over probabilistic deterministic finite automata, with Markov chain Monte Carlo (MCMC) over an infinite
mixture that averages over state cardinality. The uninformative baseline $\rho_0$ of
Appendix~\ref{sec:supp-sentence} is simpler, a uniform prior over the transition structures
with at most three states. Collapsed Gibbs and collapsed
variational methods for probabilistic finite automata \citep{collapsedpfa} marginalize
transition probabilities, and DeepDFA \citep{deepdfa} learns automata through neural
probabilistic relaxations. 

Bayesian treatments of non-Markovian rewards sit closer to control but still place the
history-dependence in the reward alone. Bayesian inverse reinforcement learning for
non-Markovian rewards \citep{birlnmr} infers reward machines from expert behaviour without
identifiability or contraction guarantees, and related work infers probabilistic reward
machines from non-Markovian reward processes \citep{inferprm}. Closest in spirit is PI-GLS
\citep{pigls}, which performs noise-robust reward machine induction with Bayesian handling of
uncertain labels and belief-aware shaping over distributions on reward machine states.

\subsection{Active automata learning}
\label{sec:rel-aal}

\paragraph{$L^*_{\mathrm{MDP}}$.} The nearest prior art is \citet{tappler2021lstar}, who learn
complete MDP structure, including states, from black-box interaction. They use schedulers, their term for behaviour
policies, computed to maximize the reachability of a target state, so they already design
policies to disambiguate, and they answer equivalence queries by directed testing. Their setting is a Markov MDP, whose semantics map test sequences to output
distributions under Myhill--Nerode equivalence over traces, and the object recovered is the
state structure of a process that is Markov given the state. $L^*_{\mathrm{MDP}}$ does not
require an omniscient teacher, since its equivalence queries are answered by directed
testing, and the implementations that ship with AALpy \citep{aalpy} answer them
statistically, through a random walk with a reset probability or through replaying an access
sequence to a hypothesis state and continuing at random. These are sound approximations. The
setting of $L^*_{\mathrm{MDP}}$ differs from ours in three respects. An equivalence test of any
accuracy must issue new traces against a live system, whereas our input is a fixed corpus
that cannot be extended.
Section~\ref{sec:eig} makes the comparison quantitative, \S\ref{sec:app-eqtest} scores the
two AALpy strategies directly, and \S\ref{sec:eig-scope} states the limits of the claim.

\paragraph{Advice and supervision.} \citet{advice2025} use advice in the form of a string
rewriting system, which is a deductive hard constraint derived from system properties such as
idempotence and commutativity. Their setting has no Bayesian framework, no posterior, and no
identifiability analysis, so it is the deductive counterpart to the
probabilistic approach taken here. \citet{bhattamishra2026} prove that the class of automata
with at most $n$ states is identifiable from positive examples together with
next-symbol-prediction labels, and they give the $L^*_{\mathrm{nsp}}$ algorithm. Their notion
of identifiability concerns labels on sequences, with no actions, no behaviour policy, and no
dynamics, whereas ours is relative to a policy.

\paragraph{Research agenda.} \citet{aalagenda2026} give a research agenda for active
automata learning arising from Dagstuhl Seminar 23492. Its framing assumes a minimally
adequate teacher throughout, and it does not treat Bayesian methods, identifiability,
non-Markovian dynamics, side information, or prior knowledge. Its \S6.1 names the gap in our
nearest prior art. The probabilities in the MDPs of \citet{tappler2021lstar} come from maximum
likelihood estimates, that is, fractions of observation counts, and do not account for
statistical error through confidence intervals. The agenda proposes to close that gap in the
frequentist way, investigating in its Challenge 27 whether approximate automata learning techniques
with PAC guarantees are possible for quantitative systems.

\subsection{Non-identifiability, misspecification, and experimental design}
\label{sec:rel-nonident}

The freezing mechanism is known in off-policy evaluation and in causal identification. In
off-policy evaluation, \citet{ope2022causal} show that under partial observability the
estimate is unavailable even with infinite observational data from the behaviour policy, and
\citet{worldmodelnonident} and \citet{ivconfounded} study related non-identifiability in world
models and in confounded Markov decision processes. In causal identification under Markov
equivalence, observational equivalence is close to definitional. Our setting is not subsumed
by these, because in a Markov MDP non-identifiability under a behaviour policy is a coverage
condition on the state--action pairs of the true model, whereas in a Regular Decision Process
it is a coverage condition on the product of the two candidates, and \S\ref{sec:collapse}
shows that covering every state of the true model does not suffice.

The supporting theory is classical. \citet{berger1986eps} introduced
$\varepsilon$-contaminated priors, which mix an elicited prior with weight $1-\varepsilon$ and
a contaminating distribution with weight $\varepsilon$, and \citet{kleijn2006misspec} give
rates under misspecification through a martingale identity. \citet{shalizi2009}
matters most here, because reinforcement learning trajectories are dependent data, and it
gives sufficient conditions for posterior convergence when every hypothesis is wrong and the
data are dependent. \citet{grunwald} shows that Bayesian inference can be inconsistent under
misspecification. Work on $\varepsilon$-contamination addresses continuous parameter spaces,
and we found no treatment of contraction over discrete structured spaces, so we use it as
machinery in Appendix~\ref{sec:supp-sentence}.

Posterior sampling, Bayes-adaptive MDPs, Bayesian Exploration Networks \citep{ben2023}, and
\citet{ross2011} maintain a posterior over models and solve the exploration--exploitation
trade-off in a Bayes-optimal way, always over parameters of Markov MDPs. BALD
\citep{bald} and Bayesian experimental design select queries by expected information gain and
are mature for classification and regression. We found no prior work combining Bayesian priors
over automaton structure with information-gain query selection, and \S\ref{sec:eig} gives our
instantiation of that combination.

\section{Additional Theory}
\label{sec:app-theory}

\subsection{Formalization and the nonzero-likelihood hypothesis}
\label{sec:lean}

Theorem~\ref{thm:frozen} is formalised in Lean~4 against Mathlib. The development proves the
odds identity, the statement that the odds stay fixed at every sample size, the corresponding
statements for the normalized posterior, and the corollary for global equivalence. It contains
no incomplete proofs, and an axiom audit confirms that every theorem depends only on
propositional extensionality, the axiom of choice, and the soundness of quotients.

\paragraph{The nonzero-likelihood hypothesis.} Theorem~\ref{thm:frozen} requires
$L_{M'}(h) \neq 0$, and the requirement is easy to overlook, because division in
$[0,\infty]$ is total and the ratio is therefore never ill-formed. Totality
of the operation does not suffice here. Where both likelihoods vanish the odds become
$0/0 = 0$, which need not equal $w/w'$, so the identity holds only where the hypothesis does.

The hypothesis is natural on its own terms. A history that both models assign probability
zero cannot be observed under either hypothesis, so it carries no evidence and says nothing
about how the odds move. Finiteness needs no such hypothesis, since likelihoods are products
of probabilities and so are at most $1$, which we also prove.

\paragraph{Two observations about the formalisation:} 

\begin{enumerate}
\item \textbf{Mathlib supports non-Markovian dynamics directly.} One might expect the
construction of the trajectory measure to dominate the cost in Lean. The trajectory
construction in Mathlib takes, at each step $n$, a kernel $\kappa_n$ from
$\prod_{i \le n} X_i$ to $X_{n+1}$, so each step's kernel already conditions on the entire
history and $X$ may vary with the timestep. History-dependence is the native shape of the Ionescu--Tulcea interface, so a
Regular Decision Process under a fixed $\pi$ is a direct instance of it. The definitional
cost lies instead in the automaton-indexed kernel and the posterior bookkeeping.

\item \textbf{Theorem~\ref{thm:frozen} needs no measure theory.} It is finite-horizon and
combinatorial, since equal per-history likelihoods give equal posterior odds, so formalising
it over finite lists of action--observation pairs keeps the proof short. The same holds for
most of the martingale development, as \S\ref{sec:martingale} describes.
\end{enumerate}

\subsection{Conservation of posterior weight}
\label{sec:martingale}

The posterior over a fixed pair of hypotheses is a bounded martingale with respect to the
data filtration, and so converges almost surely by Doob's theorem. We separate the
combinatorial part of that statement from the part that requires a measure space, because the
two carry different costs in a proof assistant and the combinatorial part carries the
substance of the argument.

\begin{proposition}[Conservation of posterior weight]
\label{prop:conservation}
For any prior weight $w$, model $M$, history $h$ and action $a$,
\[
\sum_{o \in O} w\,L_M\bigl(h\,(a,o)\bigr) \;=\; w\,L_M(h),
\]
and the same identity holds for the evidence $w\,L_M(h) + w'\,L_{M'}(h)$. Consequently the
normalized posterior satisfies the one-step tower property, and lies in $[0,1]$.
\end{proposition}

Summing posterior weight over the next observation returns the previous weight, which is the
martingale property in combinatorial form. We prove it in Lean over finite histories,
along with the one-step likelihood decomposition, the lifted statement that emission rows sum
to one, conservation of the evidence, the bound placing the posterior in $[0,1]$, and the
one-step tower property. The proof uses no measure space, no Ionescu--Tulcea construction, and no filtration. The
identity concerns finite sums over $O$, and it holds because emission rows are normalized.

\paragraph{Scope of the claim.} The statement of almost-sure convergence still requires the
trajectory measure. Proposition~\ref{prop:conservation}
establishes that every hypothesis Doob's theorem requires holds, since the process is
adapted, bounded in $[0,1]$, and satisfies the tower property. The remaining step is the
construction of the measure, and it involves no further combinatorial content. We quote the
consequences from the literature. When the truth is among the hypotheses and identifiable
under $\pi$, the posterior concentrates on it, and when every hypothesis is wrong, it
concentrates on the hypotheses closest to the truth in Kullback--Leibler divergence
\citep{berk1966,shalizi2009}. Between hypotheses that are $\pi$-equivalent,
Theorem~\ref{thm:frozen} keeps the posterior odds at their prior values.

\subsection{Structural account}

In a Markov MDP the model factorizes into conditionals indexed by $(s,a)$, so covering every
such pair identifies every parameter. In a Regular Decision Process the automaton state is
unobservable, being a deterministic function of history, so identification depends on which
combinations of history pattern and action occur. A policy can cover every state while
avoiding the joint pattern that separates the hypotheses. The always-wait policy reaches the
relevant state but never takes the action that exposes the difference there.

\begin{remark}[Two actions are necessary]
\label{rem:twoactions}
A one-action variant of this construction cannot exhibit the phenomenon. With a single action
the policy makes no choice, so it cannot leave any experiment missing. Taking one action with an emission that is
deterministic after the cue does not repair this, since the two models then disagree on
consistent histories of length at least $2$, with likelihoods $1/2$ and $1/4$, so no policy
renders them equivalent.
\end{remark}

\subsection{Proof of Theorem~\ref{thm:measurezero}}
\label{sec:app-proof}

We complete the argument sketched in \S\ref{sec:measurezero}, retaining the notation of that
section. Recall that the first entries partition $L_{B_\pi}$ into blocks
$\{\,h\,(a,o') : o' \in O\,\}$, each of mass $\Pi(h)\,\pi(h)(a)\,L_M(h)$ under either model.

Suppose first that $\Prob^\pi(L_{B_\pi}) > 0$. Some block then has positive mass, so
$\Pi(h)\,\pi(h)(a) > 0$ and $L_M(h) = L_{M'}(h) > 0$. Writing $(q,q') = \hat\Delta(h)$ and
taking $o$ to be an observation witnessing $((q,q'),a) \in B_\pi$,
\[
L_M(h\,(a,o)) = L_M(h)\,e(q,a)(o) \neq L_{M'}(h)\,e'(q',a)(o) = L_{M'}(h\,(a,o))
\]
on a history the policy generates with positive probability, so $M \not\eqpi M'$.

Suppose conversely that $M \not\eqpi M'$, and let $h\,(a,o)$ be a shortest history of positive
policy probability on which the two likelihoods differ. Minimality gives
$L_M(h) = L_{M'}(h)$, and that common value is nonzero, since otherwise both extensions would
vanish as well and the likelihoods would agree on $h\,(a,o)$. Dividing through by it yields
$e(q,a)(o) \neq e'(q',a)(o)$, so $((q,q'),a) \in B_\pi$, and $a \in A_\pi$ because $\Pi(h\,(a,o))
> 0$. The whole block over $h$ and $a$ therefore lies in $L_{B_\pi}$ and carries mass
$\Pi(h)\,\pi(h)(a)\,L_M(h) > 0$, whence $\Prob^\pi(L_{B_\pi}) > 0$. The infinite-horizon
statement follows from Proposition~\ref{prop:infinite}. \qed

\subsection{Scope conditions for the criterion}
\label{sec:app-scope}

\paragraph{Rewards as observations.} Definition~\ref{def:rdp} carries no separate reward
function, while \S\ref{sec:intro} motivates the setting by offline reinforcement learning
corpora, so we state the correspondence. An RDP whose reward is a history-dependent random
variable taking finitely many values is the case $O = O_{\mathrm{env}} \times R$ for a finite
reward set $R$, with $e(q,a)$ the joint emission over the observed pair, and a reward that is
a deterministic function of $(q,a,o)$ requires no product at all. Two of our environments are
of the first form, since Rotating MAB and Cheat MAB observe the reward and nothing else. Everything in \S\ref{sec:measurezero} concerns $e$ alone and so applies unchanged once
the reward is read as a component of the observation. Folding it in makes $\eqpi$ the finer
relation, since two models that agree on the environment observation while differing in
reward contribute a disagreeing pair to the criterion and are separable exactly when the
behaviour policy exercises that pair. The finiteness of $R$ is the same finite-alphabet
assumption we make of $O$ throughout, so continuous rewards fall outside the scope of these
results.

\paragraph{Restriction of the quantifier to $A_\pi$.} Dropping the restriction to the actions
the policy can take would place Corollary~\ref{cor:decide} in contradiction with
Proposition~\ref{prop:collapse}. Under the always-$\mathrm{wait}$ policy the T-maze reaches
one of the state pairs $(q_{\mathrm{L}}, q_{\mathrm{b}})$ and $(q_{\mathrm{R}}, q_{\mathrm{b}})$
with probability one, and at both the two emission rows differ under $\mathrm{run}$.
Collecting disagreeing pairs over all of $A$ and then testing membership by reaching such a
pair would therefore declare the two models distinguishable, whereas
Proposition~\ref{prop:collapse}, proved in Lean, establishes that they are $\pi$-equivalent,
so the posterior never moves. Since
$\mathrm{run} \notin A_\pi$ the pairs do not enter $B_\pi$, the procedure reports
equivalence, and $\Prob^\pi(L_{B_\pi}) = 0$ as Theorem~\ref{thm:measurezero} requires.

\paragraph{The trailing action.} Requiring the history to take the witnessing action at the
pair, in addition to reaching it, is a sharpening independent of the restriction to $A_\pi$.
Carrying the action makes $\Prob^\pi(L_{B_\pi})$ weight every candidate blame step by
$\pi(h)(a)$, which is why Theorem~\ref{thm:measurezero} needs neither
Assumption~\ref{asm:support} nor the restriction of $B_\pi$ to $A_\pi$, since an action the
policy never takes contributes no mass on its own. Under Assumption~\ref{asm:support} the
two readings coincide, because if $a \in A_\pi$ and the product automaton reaches $(q,q')$
after a history of positive probability then the policy takes $a$ there with positive
probability, so reaching the pair and exercising it there have the same null sets. That
coincidence permits Algorithm~\ref{alg:decide} to test the
pair alone. A policy whose support varies with history separates the two readings, and the
reading with the trailing action is the one that remains correct.

\begin{remark}[The two hypotheses the theorem requires]
\label{rem:hyp}
Theorem~\ref{thm:measurezero} requires normalized emission rows and the ordering in which the
action precedes the observation. Normalization collapses the sum over $o'$ within a
first-entry block to the mass of its prefix, and the action-before-observation ordering keeps
the policy factor $\Pi(h)\,\pi(h)(a)$ outside that sum, so that no policy can shift mass
between the two models at the blame step. Both belong to Definition~\ref{def:rdp} and to the
factorization in \S\ref{sec:defs}. A formalism with simultaneous action and observation
timing therefore requires a check of this correspondence, since dropping either hypothesis
breaks the argument.
\end{remark}

\paragraph{Ambiguity of the measure.} The phrase ``measure zero under $\Prob^\pi$'' would be
ambiguous over an arbitrary set of histories, since a history that $M$ calls impossible may
be one that $M'$ allows, and such a history has $M$-measure zero while witnessing that the
models are not $\pi$-equivalent. Over $L_{B_\pi}$ the ambiguity does not arise, because every
first-entry block carries the same mass
$\Pi(h)\,\pi(h)(a)\,L_M(h) = \Pi(h)\,\pi(h)(a)\,L_{M'}(h)$ under both models, so
$\Prob^\pi_M(L_{B_\pi}) = \Prob^\pi_{M'}(L_{B_\pi})$ and the subscript may be dropped. The
disagreement set of \S\ref{sec:app-conjecture} has neither property, and the
compensating-history argument that removes the ambiguity there is given in that section.

\begin{remark}[A caveat on the policy]
\label{rem:policycaveat}
Definition~\ref{def:support} quantifies $A_\pi$ over every history the policy can generate,
whereas the implementation probes the support at the empty history alone. The two agree
exactly when the action support does not vary with history, which holds for every policy in
our suite, and we verify it over all histories up to a fixed length. A policy whose support
varies requires folding the support into the product automaton, which
Appendix~\ref{sec:supp-algo} describes together with the regularity condition it needs and its
cost.
\end{remark}

\subsection{The conjecture over the disagreement set}
\label{sec:app-conjecture}

The natural conjecture concerns the \emph{disagreement set}
\[
D(M,M') \;=\; \{\, h \in \Hist \;:\; L_M(h) \neq L_{M'}(h) \,\},
\]
and it asserts that $M \eqpi M'$ iff $\Prob^\pi(D) = 0$. Treating $D$ as tractable would
require it to be regular, and a product construction on the two automata appears to show
that it is.

The product construction does not establish this, because $L_M(h)$ is a product of real
numbers. Two histories that drive the product automaton into unrelated state pairs can still
have equal likelihoods by arithmetic coincidence, as when a factor of
$\tfrac{1}{2}\cdot\tfrac{2}{3}$ on one path matches $\tfrac{1}{3}\cdot 1$ on another.
Membership in $D$ therefore depends on the numerical value of a product, which no
finite-state property of $h$ determines, so in general no finite automaton recognizes it. The
product construction proves the regularity of a different set from the one $D$ names. The
coinciding parameter values are not pathological, since nothing distinguishes them from any
others.

\paragraph{Readings of the conjecture.} The biconditional the conjecture asserts is true. The
phrase $\Prob^\pi(D) = 0$ admits three readings, namely (A) $\Prob^\pi_M(D) = 0$,
(B) $\Prob^\pi_M(D) = 0$ and $\Prob^\pi_{M'}(D) = 0$ together, and (C) $\pi$-equivalence
itself. At finite horizon all three coincide, so the conjecture holds under each. That (B)
implies (C) is immediate. That the weaker reading (A) implies it as well is the content of the
compensating-history argument, which runs as follows. Let $h\,(a,o)$ be a shortest history of
positive policy probability on which the likelihoods differ, so that
$L_M(h) = L_{M'}(h) \neq 0$ and the two emission rows at $\hat\Delta(h) = (q,q')$ differ under
$a$. Those rows are normalized, so some $o'$ has $e(q,a)(o') > e'(q',a)(o') \geq 0$ and in
particular $e(q,a)(o') > 0$. The compensating history $h\,(a,o')$ then lies in $D$ and carries
positive $M$-mass, and because the policy picks $a$ before seeing that step's observation the
factor $\Pi(h)\,\pi(h)(a)$ is common to both models, so no policy can suppress it.

We checked this on 4500 model pairs, 3000 random and 1500 constructed to be
$\pi$-equivalent, and the three readings agreed on every pair. We also checked the
compensating-history argument itself, confirming that the intended witness lies in $D$ with
zero mass while the compensating history lies in $D$ with positive mass, as the argument
requires.

The conjecture is therefore true but of no use for deciding equivalence. Because $D$ is not
regular, the conjecture concerns an object with no finite-state description, so it certifies
nothing about a given pair, yields no decision procedure, and cannot be carried to infinite
horizon through the regularity of $D$. The null-set arguments that such an extension appears
to require arise from the choice of object. Section~\ref{sec:measurezero} therefore changes the
object, asking the same question of $L_{B_\pi}$, which is regular by construction, and
Proposition~\ref{prop:infinite} then settles infinite horizon by countable additivity alone.

\subsection{Extension to infinite horizon}

The infinite-horizon case requires no separate treatment of null sets under the trajectory
measure.

\begin{proposition}[Infinite horizon]
\label{prop:infinite}
Write $E_{B_\pi}$ for the event that some prefix of an infinite trajectory lies in
$L_{B_\pi}$. Then $E_{B_\pi} = \bigcup_{h \in L_{B_\pi}^{\min}} \mathrm{Cyl}(h)$, where
$L_{B_\pi}^{\min}$ is the countable set of first entries of $L_{B_\pi}$. Hence
$\Prob^\pi(E_{B_\pi}) = 0$ iff $\Prob^\pi(\mathrm{Cyl}(h)) = 0$ for every such $h$, by
countable additivity alone.
\end{proposition}

Beyond the existence of the trajectory measure, which the Ionescu--Tulcea construction
provides, the argument needs neither a martingale argument nor a regularity condition on the
disagreement set. The restatement in Definition~\ref{def:blame} makes this possible, since it
replaces a set defined by arithmetic with one defined by reachability.

\section{Additional Experimental Results}
\label{sec:app-exp}

\subsection{The two inference paths}
\label{sec:app-ml}

This paper uses two likelihoods, which serve different results.

\paragraph{Fixed emissions, exact.} The identifiability and design results take a \emph{pair}
of fully specified candidates whose emission tables are given rationals. The likelihood is
$L_M(h)$ of Definition~\ref{def:histlik}, the posterior over the pair is the normalized
product of prior and likelihood, and nothing is integrated out. This path produces the
separating masses of Definition~\ref{def:sepmass}, the expected information gains of
\S\ref{sec:eig}, and the log Bayes factors of Table~\ref{tab:pipeline} and
Table~\ref{tab:depth}.

\paragraph{Integrated emissions, Dirichlet.} The structure-level results take the emissions
to be unknown and the transition function to be the object of inference. Placing an
independent symmetric $\mathrm{Dirichlet}(\alpha)$ prior on each emission row and integrating
gives, for a corpus $x$ with counts $n_{q,a}(o)$ of observation $o$ at the pair $(q,a)$ and
$N_{q,a} = \sum_{o} n_{q,a}(o)$,
\[
P(x \mid \delta) \;=\; \prod_{(q,a)}
\frac{(k\alpha - 1)!}{\bigl((\alpha - 1)!\bigr)^{k}} \cdot
\frac{\prod_{o \in O} \bigl(n_{q,a}(o) + \alpha - 1\bigr)!}{\bigl(N_{q,a} + k\alpha - 1\bigr)!},
\qquad k = |O|,
\]
which for integer $\alpha$ is a ratio of factorials and so an exact rational, keeping the
structure posterior rational. We use $\alpha = 1$ throughout and sweep it over
$\{1,2,3,5,10\}$ in \S\ref{sec:ablations}. This path produces the Bayes factors against a
memoryless null in Table~\ref{tab:bf}, the Tomita panel of Figure~\ref{fig:nonmarkov}, and
the priors of Appendix~\ref{sec:supp-sentence}. Where a point estimate of an emission row is
needed, as in Table~\ref{tab:stoch}, it is the posterior mean
$\bigl(n_{q,a}(o) + \alpha\bigr) / \bigl(N_{q,a} + k\alpha\bigr)$ under the same prior.

The design objective of \S\ref{sec:eig} is stated over the first path, because the question
\textsc{Pec} decides concerns a given pair of candidates. Which structure a corpus supports is
the question of the second path, and only the results listed above use it.

\subsection{Structural depth in the corridor family}
\label{sec:app-corridor}

Table~\ref{tab:tier1c} reports the separating horizon that \textsc{Pec} measures for each
member of the family of \S\ref{sec:tier1c}, against the depth the member was built to have.

\begin{table}[!ht]
\centering\small
\caption{Corridor depth against the measured separating horizon. A horizon of $\bot$ means
that \textsc{Pec} found no reachable disagreeing pair, so no finite horizon separates the two
models and the entry is a certificate of $\pi$-equivalence.}
\label{tab:tier1c}
\begin{tabular}{lcccccccc}
\toprule
requested depth & 1 & 2 & 3 & 4 & 5 & 6 & 8 \\
\midrule
rooms & 1 & 2 & 3 & 4 & 5 & 6 & 8 \\
$\pi$-equivalent & yes & yes & yes & yes & yes & yes & yes \\
horizon under $\pi$ & $\bot$ & $\bot$ & $\bot$ & $\bot$ & $\bot$ & $\bot$ & $\bot$ \\
horizon under deviation & 1 & 2 & 3 & 4 & 5 & 6 & 8 \\
\bottomrule
\end{tabular}
\end{table}

Table~\ref{tab:lookahead} gives the lookahead sweep of \S\ref{sec:tier1c}.

\begin{table}[!ht]
\centering\small
\caption{Separating mass achieved by exhaustive design at lookahead $\ell$.}
\label{tab:lookahead}
\begin{tabular}{lcccccc}
\toprule
depth $d$ & $\ell=1$ & $\ell=2$ & $\ell=3$ & $\ell=4$ & $\ell=5$ & $\ell=6$ \\
\midrule
$1$ & $1$ & $1$ & -- & -- & -- & -- \\
$2$ & $0$ & $1$ & $1$ & -- & -- & -- \\
$3$ & $0$ & $0$ & $1$ & $1$ & -- & -- \\
$4$ & $0$ & $0$ & $0$ & $1$ & $1$ & -- \\
$5$ & $0$ & $0$ & $0$ & $0$ & $1$ & $1$ \\
\bottomrule
\end{tabular}
\end{table}

\subsection{End-to-end comparison at shallow depth}
\label{sec:app-pipeline}

Table~\ref{tab:pipeline} gives the corpus comparison of \S\ref{sec:method-negative} in full.

\begin{table}[!ht]
\centering\small
\caption{One learner at a fixed budget on three corpora, as mean $\pm$ half-range over ten
seeds. BF is the Bayes factor and LL the log-likelihood. The zeros under $\pi$ are exact
rational zeros, so their spread is exactly zero by Theorem~\ref{thm:frozen}.}
\label{tab:pipeline}
\begin{tabular}{llrrr}
\toprule
Environment & corpus & $\log_{10}$ BF & coverage & LL/step$^*$ \\
\midrule
T-maze        & under $\pi$ & $\mathbf{0.00 \pm 0.00}$ & $16.2 \pm 3.2$\% & $-0.6932 \pm 0.0014$ \\
              & random      & $300.91 \pm 9.03$ & $99.9 \pm 0.3$\% & $-0.4303 \pm 0.0449$ \\
              & designed    & $300.91 \pm 9.03$ & $99.9 \pm 0.3$\% & $-0.4303 \pm 0.0449$ \\
\addlinespace
Rotating MAB  & under $\pi$ & $\mathbf{0.00 \pm 0.00}$ & $8.4 \pm 1.7$\% & $-0.6932 \pm 0.0008$ \\
              & random      & $161.93 \pm 20.04$ & $100.0 \pm 0.0$\% & $-0.6798 \pm 0.0082$ \\
              & designed    & $148.86 \pm 17.89$ & $29.5 \pm 3.3$\% & $-0.6851 \pm 0.0070$ \\
\addlinespace
Cheat MAB     & under $\pi$ & $\mathbf{0.00 \pm 0.00}$ & $29.5 \pm 3.3$\% & $-0.6553 \pm 0.0173$ \\
              & random      & $99.48 \pm 16.46$ & $100.0 \pm 0.0$\% & $-0.6378 \pm 0.0222$ \\
              & designed    & $99.48 \pm 16.46$ & $100.0 \pm 0.0$\% & $-0.6378 \pm 0.0222$ \\
\addlinespace
Rotating Maze & under $\pi$ & $\mathbf{0.00 \pm 0.00}$ & $16.2 \pm 3.2$\% & $-0.6741 \pm 0.0104$ \\
              & random      & $87.41 \pm 13.12$ & $100.0 \pm 0.0$\% & $-0.6301 \pm 0.0285$ \\
              & designed    & $87.41 \pm 13.12$ & $100.0 \pm 0.0$\% & $-0.6301 \pm 0.0285$ \\
\bottomrule
\end{tabular}

\vspace{3pt}
\begin{minipage}{0.86\linewidth}
\footnotesize $^*$ Held-out log-likelihood per step, taken over \emph{every} held-out step
and charging $\log(1/|O|)$ for each step the model cannot score, so that all three rows of an
environment are divided by the same denominator. The per-covered-step figure is not
comparable across rows, since it conditions on whichever subset each model happens to cover
and so flatters a corpus that covers only the easy steps.
\end{minipage}
\end{table}

The random and designed rows are identical on three of the four environments because there
the deviation assigns equal probability to every action. On the T-maze and Rotating Maze it
mixes the two actions half and half, and on Cheat MAB it plays each of the three arms with
probability $1/3$. The two rows therefore describe the same policy and draw the same corpus
under a shared seed. Rotating MAB is the one environment where they differ, and there the
difference is smaller than either row's spread over seeds, so the two corpora tie. The
coverage column separates them. The designed corpus concentrates on the missing experiment
and covers $29.5$\% of held-out steps against uniform exploration, which is the
cost of design at a depth where design brings no gain.

\subsection{Evidence against memorylessness}
\label{sec:bf}

This comparison places the true automaton against the memoryless null, a single-state model that
serves as a null hypothesis. Large Bayes factors establish
that the environments are non-Markovian and that the posterior detects it, which is a
precondition for everything that follows. Table~\ref{tab:bf} reports the comparison at 40
trajectories of 8 steps, where every environment rejects memorylessness.

\begin{table}[!ht]
\centering
\caption{Log Bayes factor of the true automaton against a memoryless null.}
\label{tab:bf}
\begin{tabular}{lrlr}
\toprule
Environment & States & Non-Markovian structure & $\log_{10}$ BF \\
\midrule
T-maze & 3 & remember the initial cue & $35.06$ \\
Rotating MAB & 2 & reward probabilities rotate on every reward & $8.03$ \\
Cheat MAB & 3 & an action sequence unlocks max reward & $7.81$ \\
Rotating Maze & 6 & orientation rotates every 3 moves & $9.84$ \\
\bottomrule
\end{tabular}
\end{table}

\begin{figure}[!ht]
\centering
\begin{subfigure}[t]{0.48\textwidth}
  \centering
  \includegraphics[width=\linewidth]{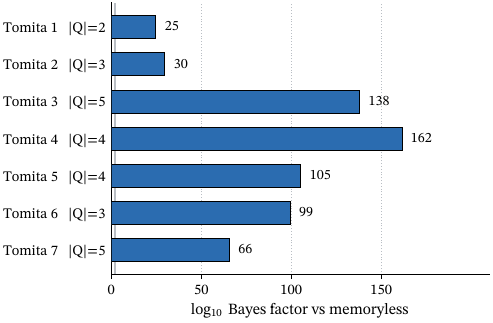}
  \caption{Evidence against memorylessness.}
  \label{fig:nm-evidence}
\end{subfigure}\hfill
\begin{subfigure}[t]{0.48\textwidth}
  \centering
  \includegraphics[width=\linewidth]{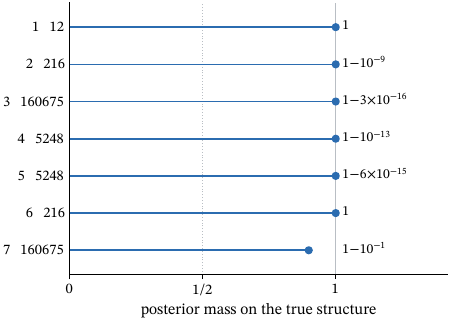}
  \caption{Posterior mass on the true structure.}
  \label{fig:nm-tomita}
\end{subfigure}
\caption{Rejection of memorylessness and structure recovery on the Tomita grammars.}
\label{fig:nonmarkov}
\end{figure}

Figure~\ref{fig:nonmarkov} reports the same comparison on the seven Tomita grammars. Every
grammar clears the threshold for decisive evidence by at least an order of magnitude, and the
structure at the posterior mode is the true one in all seven cases. The Tomita languages are
acceptors, so we encode each as an RDP with a single action $a$. The minimal DFA supplies the
transition structure over $O = \{0,1\}$, its acceptance labelling is discarded, and the state
$q_i$ of an $n$-state automaton, for $i = 0,\dots,n-1$, emits $0$ with probability
$(i+1)/(n+1)$, so no two states share an emission row. The corpus is 60 strings of 40
symbols, drawn by emitting $o \sim e(q,a)$ and moving to $\delta(q,a,o)$. For the five-state grammars (Tomita 3 and 7) we enumerate one
canonical representative per state-renaming class, $160{,}675$ structures in place of
$5^{10}$. With a single action there is no behaviour policy and no choice to decline, so this
panel tests the structure-level machinery alone and carries no policy-relative content.

\subsection{Recovery of emission probabilities}

\begin{table}[!ht]
\centering
\caption{Posterior-mean emission probabilities against the truth.}
\label{tab:stoch}
\begin{tabular}{lrrrr}
\toprule
Environment & Rows visited & Mean abs.\ err & Max abs.\ err & Min samples/row \\
\midrule
T-maze & 6/6 & $0.031$ & $0.060$ & 17 \\
Rotating MAB & 6/6 & $0.030$ & $0.067$ & 41 \\
Cheat MAB & 9/9 & $0.051$ & $0.167$ & 20 \\
Rotating Maze & 12/12 & $0.077$ & $0.365$ & 7 \\
\bottomrule
\end{tabular}
\end{table}

The maximum error tracks the sparsest row, which indicates sampling error. On Rotating Maze the mean absolute error falls from
$0.0769$ to $0.0281$, $0.0151$, and $0.0095$, and the maximum error falls from $0.365$ to
$0.120$, $0.048$, and $0.022$, as the number of trajectories goes from $40$ to $160$, $640$,
and $2560$, so the maximum error falls by a factor of about $17$.

\begin{figure}[!ht]
\centering
\begin{subfigure}[t]{0.48\textwidth}
  \centering
  \includegraphics[width=\linewidth]{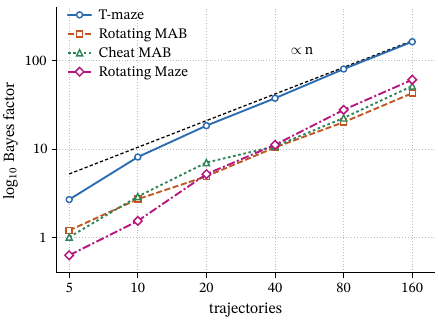}
  \caption{Evidence against trajectory count.}
  \label{fig:sc-evidence}
\end{subfigure}\hfill
\begin{subfigure}[t]{0.48\textwidth}
  \centering
  \includegraphics[width=\linewidth]{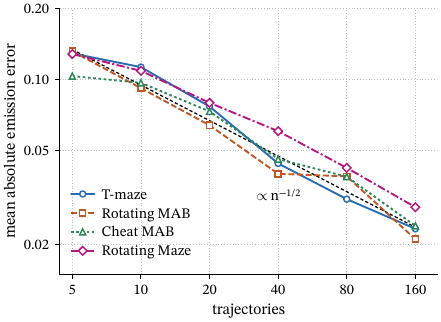}
  \caption{Emission error against trajectory count.}
  \label{fig:sc-error}
\end{subfigure}
\caption{Scaling of evidence and emission error, averaged over five seeds.}
\label{fig:scaling}
\end{figure}

Figure~\ref{fig:scaling} shows both rates on logarithmic axes. Evidence for
history-dependence runs parallel to a slope-one reference, so the log Bayes factor grows close
to linearly in the number of trajectories, rising on the T-maze from $2.70$ to $163.47$
between 5 and 160 trajectories. Emission error follows the $n^{-1/2}$ reference, which rules
out a systematic failure on the sparsest row.

\begin{remark}[A caveat on identifiability]
Structure recovery presupposes that the automaton states differ in behaviour. Two states
carrying the same emission distribution are merged under any amount of data.
\end{remark}

\subsection{Ablations}
\label{sec:ablations}

Tables~\ref{tab:bf} and~\ref{tab:stoch} report a single seed. Table~\ref{tab:seeds} gives
their variation over 12 seeds, and the minimum over seeds remains decisive in every
environment. Figure~\ref{fig:robust} plots the
quantities of this subsection.

\begin{table}[!ht]
\centering
\caption{Variation of the main quantities over 12 seeds. BF: Bayes factor; MAE: mean absolute
error; sd: standard deviation.}
\label{tab:seeds}
\begin{tabular}{lrrr}
\toprule
Environment & $\log_{10}$ BF mean $\pm$ sd & min $\log_{10}$ BF over seeds & emission MAE mean $\pm$ sd \\
\midrule
T-maze & $37.36 \pm 2.07$ & $34.82$ & $0.044 \pm 0.015$ \\
Rotating MAB & $10.78 \pm 3.09$ & $6.27$ & $0.045 \pm 0.018$ \\
Cheat MAB & $11.12 \pm 3.29$ & $6.37$ & $0.051 \pm 0.016$ \\
Rotating Maze & $10.41 \pm 1.99$ & $8.47$ & $0.065 \pm 0.014$ \\
\bottomrule
\end{tabular}
\end{table}

\paragraph{Prior strength.} Sweeping the Dirichlet concentration over
$\alpha \in \{1,2,3,5,10\}$, the log Bayes factor declines steadily and never changes sign. It
falls from $35.1$ to $24.4$ on the T-maze, from $8.0$ to $6.8$ on Rotating MAB, from $7.8$ to
$5.3$ on Cheat MAB, and from $9.8$ to $6.2$ on Rotating Maze.

\paragraph{Small-data crossover.} At 2 trajectories, 7 of the 12 seeds on the T-maze favour
the memoryless null, which is the correct preference for a simpler model, since two
trajectories do not justify three states. The T-maze is the only environment in which a
strict majority of seeds does so, with 6 of 12 on Cheat MAB. By 10 trajectories 1 of the 48
seeds still favours the null, and every mean is positive.

\paragraph{Deviation strength.} Section~\ref{sec:deviate} reports this sweep, where the
non-monotone behaviour of Cheat MAB is the informative case.

\subsection{Comparison with the reachability criterion}
\label{sec:eig-scope}

On the T-maze, reachability assigns every candidate deviation the same score. The cue is
drawn on the first step whatever the agent does, so the behaviour policy already visits every
automaton state, and the reachability score stays constant across all candidate deviations,
which leaves the criterion with no preference among them.

This blindness does not extend to every non-Markovian problem, and the four environments of
the suite split into two groups (Table~\ref{tab:reach}).

\begin{table}[!ht]
\centering\small
\caption{Form of the separating evidence in each environment.}
\label{tab:reach}
\begin{tabular}{lll}
\toprule
Environment & Separating evidence is\ldots & Reachability \\
\midrule
T-maze, Rotating MAB & an action taken in an \textbf{already-visited} state & \textbf{blind} (score constant) \\
Cheat MAB, Rotating Maze & a \textbf{state the policy never reaches} & informative \\
\bottomrule
\end{tabular}
\end{table}

When the missing experiment is to reach a state $q$, reachability is the right criterion,
which is the case $L^*_{\mathrm{MDP}}$ addresses. It gives no signal only when the policy
already visits every automaton state and the distinguishing evidence is which action was
taken there. Reachability is therefore complete for the state-reachability subclass and
uninformative outside it, and the T-maze shows that the complement of that subclass is
nonempty. Expected information gain distinguishes the deviations
from the behaviour policy on all four environments, taking the value zero under the
behaviour policy and a positive value under every deviation.

\subsection{Dependence on deviation strength}
\label{sec:deviate}

\paragraph{Definition of $p$.} The deviation strength $p$ is the probability the
deviating policy assigns to the one action the behaviour policy declines, with the remaining
mass spread uniformly over the other actions.

\begin{table}[!ht]
\centering\small
\caption{The deviation $\pi'_p$ in each environment. At $p = 0$ the deviation is the
behaviour policy, and $p$ is the probability of the named action at every step.}
\label{tab:devparam}
\begin{tabular}{llll}
\toprule
Environment & Behaviour policy $\pi$ & Action taken w.p.\ $p$ & Remaining mass \\
\midrule
T-maze        & always $\mathrm{wait}$ & $\mathrm{run}$  & $\mathrm{wait}$, w.p.\ $1-p$ \\
Rotating MAB  & always $\mathrm{arm0}$ & $\mathrm{arm1}$ & $\mathrm{arm0}$, w.p.\ $1-p$ \\
Cheat MAB     & never $\mathrm{arm2}$  & $\mathrm{arm2}$ & $\mathrm{arm0},\mathrm{arm1}$, w.p.\ $(1-p)/2$ each \\
Rotating Maze & always $\mathrm{stay}$ & $\mathrm{move}$ & $\mathrm{stay}$, w.p.\ $1-p$ \\
\bottomrule
\end{tabular}
\end{table}

The value used everywhere else in the paper is $p = 1/2$, except on Cheat MAB where it is
$p = 1/3$, so that the three arms are exercised comparably. The sweep below varies $p$ over
its whole range on Cheat MAB, which is the environment where the dependence is not monotone.

The choice of deviation matters for a second reason, independent of the comparison with
$L^*_{\mathrm{MDP}}$. Sweeping the deviation probability $p$ from $0$ to $1$ on Cheat MAB
gives the following:

\begin{table}[!ht]
\centering
\caption{Separating mass on Cheat MAB against deviation strength $p$.}
\label{tab:cheat}
\begin{tabular}{lccccccccc}
\toprule
$p$ & $0$ & $1/8$ & $1/4$ & $3/8$ & $1/2$ & $5/8$ & $3/4$ & $7/8$ & $1$ \\
\midrule
separating mass & $\mathbf{0}$ & $7/128$ & $3/32$ & $15/128$ & $\mathbf{1/8}$ & $15/128$ & $3/32$ & $7/128$ & $\mathbf{0}$ \\
\bottomrule
\end{tabular}
\end{table}

\begin{figure}[!ht]
\centering
\begin{subfigure}[t]{0.48\textwidth}
  \centering
  \includegraphics[width=\linewidth]{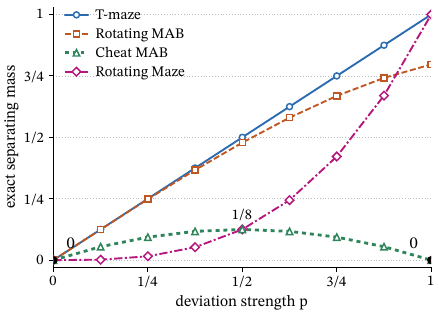}
  \caption{Separating mass.}
  \label{fig:dev-mass}
\end{subfigure}\hfill
\begin{subfigure}[t]{0.48\textwidth}
  \centering
  \includegraphics[width=\linewidth]{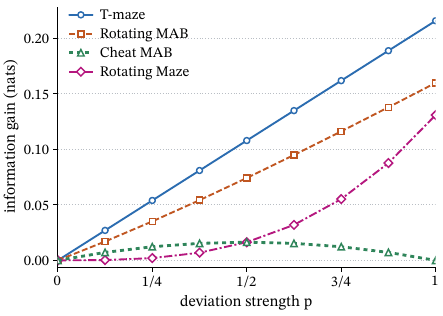}
  \caption{Expected information gain.}
  \label{fig:dev-gain}
\end{subfigure}
\caption{Deviation strength swept from the behaviour policy to its opposite.}
\label{fig:deviation}
\end{figure}

Figure~\ref{fig:deviation} shows the separating mass and the expected information gain over
the same sweep in all four environments.

The cheat sequence is $(\mathrm{arm1}, \mathrm{arm2})$. A policy that always plays
$\mathrm{arm2}$ therefore never plays $\mathrm{arm1}$, never completes the sequence, and never
unlocks the payoff. A heuristic that explores more arrives at $p=1$ and learns
nothing, because only the correct history pattern separates the models. On the other three
environments the sweep is monotone, which makes Cheat MAB the informative case.

\subsection{\textsc{Pec} under estimated emission tables}
\label{sec:app-robust}

The identifiability results of this paper are computed from exact rational emission tables,
and Definition~\ref{def:blame} accordingly tests emission rows for exact inequality. In
practice the tables are estimated from counts, and two rows estimated from the same
distribution differ with probability approaching one as the corpus grows. The exact test then
reports a disagreeing pair with probability approaching one, so \textsc{Pec} almost never
certifies $\pi$-equivalence on estimated tables. The certificate is the main output of the
procedure, so this is the respect in which the results here depend on exactness.

\paragraph{The tolerant test.} Section~\ref{sec:scale} replaces the test in
Definition~\ref{def:blame} by $\max_o |e(q,a)(o) - e'(q',a)(o)| > \tau$. With $\eta$ the
sup-norm error of the estimated tables over the rows \textsc{Pec} consults and $\gamma$ the
smallest sup-norm separation among pairs whose exact rows disagree, the tolerant test returns
the exact-arithmetic verdict on every pair whenever
\[
2\eta \;<\; \tau \;<\; \gamma - 2\eta ,
\]
a window that is non-empty as soon as $\eta < \gamma/4$ and that opens at a finite sample size.

\begin{table}[!ht]
\centering\small
\caption{\textsc{Pec} on estimated emission tables, corridor family, 600 trajectories per
candidate under a covering policy, five seeds. $\eta$ is measured over the rows the procedure
consults.}
\label{tab:robust}
\begin{tabular}{rccccc}
\toprule
depth $d$ & $\gamma$ & $\eta$ & window & $\tau = 0$ & $\tau$ in window \\
\midrule
$2$ & $1/2$ & $0.072$ & $(0.144,\,0.356)$ & never certifies & all correct \\
$4$ & $1/2$ & $0.083$ & $(0.167,\,0.333)$ & never certifies & all correct \\
$6$ & $1/2$ & $0.098$ & $(0.195,\,0.305)$ & never certifies & all correct \\
$8$ & $1/2$ & $0.095$ & $(0.190,\,0.310)$ & never certifies & all correct \\
\bottomrule
\end{tabular}
\end{table}

In Table~\ref{tab:robust}, at $\tau = 0$ the estimated rows never agree, so the exact test is
unusable on estimates. At every $\tau$ strictly inside the window, sampled at eight points per
row, the tolerant test reproduces exact arithmetic, certifying $\pi$-equivalence under the
behaviour policy and still returning the witness under the deviation. At
$\tau = \gamma + 2\eta$, above the window, it certifies the deviation as well, which is a false
certificate, so the upper bound of the window cannot be dropped.

\paragraph{Choice of estimation policy.} Estimating the tables under a uniform policy makes
$\eta$ grow with depth, reaching $0.19$ at $d = 8$ and closing the window. The cause is the
structure of the corridor. Reaching room $d-1$ requires $d-1$ waits in succession, so under
uniform exploration the deep rooms are visited a $2^{-(d-1)}$ fraction of the time and their
rows are estimated worst, and they are exactly the rows that carry the missing experiment.
Weighting the estimation policy towards traversal restores the window at every depth. The same
observation motivates designing the experiment, and here it applies to the estimation of the
tables.

\paragraph{Near-equivalent candidates.} The window closes as $\gamma$ approaches zero. For any
fixed $\gamma > 0$ it still opens at a finite sample size, although that size grows without
bound as the candidates approach equivalence. The posterior likewise still moves, more slowly
as the separation decreases. This gradual degradation is the quantitative counterpart of
Theorem~\ref{thm:frozen}. A continuous observation space enters this account the same way,
through state aliasing that shrinks $\gamma$ (Appendix~\ref{sec:limitations}).

\subsection{Equivalence testing as an exploration strategy}
\label{sec:app-eqtest}

Section~\ref{sec:method-negative} compares designed experiments against uniform
exploration. Modern active automata learning explores better than uniformly, and
Table~\ref{tab:eqtest} measures by how much. AALpy \citep{aalpy} ships two standard strategies
for answering equivalence queries, and both decompose into ingredients evaluated elsewhere in
this paper. The random-walk strategy explores by uniform random actions with a reset
probability, which is the random corpus of \S\ref{sec:method-negative}. The state-prefix
strategy replays an access sequence to a hypothesis state and then continues at random, and
its first ingredient is the state-reachability criterion of \S\ref{sec:eig-scope}. We score all
three strategies by the exact probability of producing a history on which the candidates
disagree.

\begin{table}[!ht]
\centering\small
\caption{Separating mass of each exploration strategy on the corridor family, as exact
rationals.}
\label{tab:eqtest}
\begin{tabular}{rcccr}
\toprule
depth $d$ & random walk & state prefix & designed & prefix / walk \\
\midrule
$2$ & $1/4$   & $3/8$      & $1$ & $1.5\times$ \\
$4$ & $1/16$  & $15/64$    & $1$ & $3.8\times$ \\
$6$ & $1/64$  & $21/128$   & $1$ & $10.5\times$ \\
$8$ & $1/256$ & $255/2048$ & $1$ & $\mathbf{31.9\times}$ \\
\bottomrule
\end{tabular}
\end{table}

On this family the access-sequence prefix improves on the random walk, because the missing
experiment sits behind a state the walk has to reach. The prefix converts an exponential decay
into a polynomial one, and the state-prefix strategy beats the random walk by a factor that
grows to $32$ at depth eight. The polynomial is $1/d$, since the product of depth and
separating mass is $0.750$, $0.938$, $0.984$ and $0.996$ down the column. The rate follows
from the strategy, which cycles through $d$ hypothesis states and so spends a $1/d$ fraction
of its budget on the room that carries the experiment, probing there with probability $1/2$ at
the first opportunity and with the remaining mass later. A designed deviation executes the
separating sequence every time and has separating mass $1$ at every depth, so the state-prefix
strategy still falls short of the experiment. Table~\ref{tab:eqtest} therefore
places the gap between designed experiments and a competent active-learning heuristic below
the gap against uniform exploration reported in Table~\ref{tab:depth}, without closing it.

\begin{table}[!ht]
\centering\small
\caption{The same strategies on the T-maze, where every automaton state is visited on the
first step whatever the agent does.}
\label{tab:eqtest-tmaze}
\begin{tabular}{lccc}
\toprule
& random walk & state prefix & designed \\
\midrule
separating mass & $1/2$ & $1/2$ & $1$ \\
\bottomrule
\end{tabular}
\end{table}

Table~\ref{tab:eqtest-tmaze} isolates the prefix ingredient. On the T-maze the cue is drawn
on the first step whatever the agent does, so every automaton state is already visited, no
access sequence changes what the strategy sees, and the state-prefix strategy collapses exactly
onto its random-walk tail, with the two entries equal as rationals. This is the blindness of
the reachability criterion recorded in \S\ref{sec:eig-scope}, here measured on a named
implementation. On this subclass the reachability ingredient contributes nothing, and whatever
the strategy achieves it achieves by walking at random.

\paragraph{Scope of the comparison.} These are exploration strategies scored on a fixed pair
of candidates, and the comparison does not replicate $L^*_{\mathrm{MDP}}$ or AALpy's learning
loop. Both strategies exist to answer conjectures against a live system inside an active
learning algorithm, and neither is available in the offline setting this paper addresses,
where the corpus cannot be extended. The measurement is meaningful because the ingredient
that does the work in an equivalence test, the choice of where to go before acting at random,
is the quantity our design step chooses, so the two can be compared without running either
algorithm outside its assumptions.

\subsection{Computation of RegORL's distinguishability parameter}
\label{sec:app-regorl}

Section~\ref{sec:rel-offline} argues that \textsc{Pec} and RegORL \citep{regorl} produce
objects that are not commensurable as estimators, although the parameter on which the RegORL
guarantee rests can be computed exactly on our environments.

\paragraph{The parameter.} RegORL's guarantee rests on its Assumption~2, that the behaviour
policy has $L^p_\infty$-distinguishability at least $\mu_0 > 0$. For a \emph{minimal} RDP with
layered states $Q = \bigcup_t Q_t$, $\mu_0$ is the largest value such that, for every $t$ and
every pair of distinct $q, q' \in Q_t$, the $L^p_\infty$ distance between the distributions of
the trace from step $t$ onwards, started in $q$ and in $q'$ under $\pi$, is at least $\mu_0$.
That distance is the largest difference, over trace prefixes within the remaining horizon,
between the probabilities that the two traces begin with the prefix. In the notation of
\citet{regorl}, where $Q$, $A$, $O$ and $R$ denote the numbers of states, actions,
observations and rewards, $H$ the horizon, $\delta$ the failure probability and $d^b_{\min}$
their concentrability term, their Theorem~6 requires a corpus of size
\[
N_\delta \;=\; 21\log(8QAO/\delta)\sqrt{H\log(2ARO)} \,/\, (d^b_{\min}\mu_0),
\]
so $\mu_0$ enters in the denominator at the first power and the requirement diverges as it
goes to zero.

Two features of the definition matter below. The probability of a trace prefix includes the
policy's own action probabilities, so an action $\pi$ never takes contributes zero to both
sides and separates nothing. The quantifier also ranges over states of the \emph{minimal} RDP,
minimal in the Moore sense, so two states of a layer that no input sequence separates within
the steps the episode has left count as one state. That quotient is taken with respect to the
dynamics, independently of $\pi$, and it cannot be skipped, since without it every
fixed-horizon RDP would report $\mu_0 = 0$ from its last layer, where nothing has time to
separate, and the parameter would say nothing about any policy.

\paragraph{Evaluation on the four environments.} Table~\ref{tab:ident} reports $\mu_0$ on the
four environments. On three of them $\mu_0 = 0$ as an exact rational, so Assumption~2 fails,
$N_\delta$ diverges, and the RegORL guarantee is vacuous, and on each of them the separating
mass is also an exact rational zero. Two of those three are benchmarks from
\citet{abadi2020}. The state pair that attains the minimum is $q_{\mathrm{L}}$ against
$q_{\mathrm{R}}$ at $t = 2$ on the T-maze, and the two automaton states at $t = 2$ on each MAB.

On Rotating Maze the assumption instead holds \emph{trivially}. The behaviour policy never
moves, so it reaches a single state per layer, no pair of distinct states exists to be told
apart, and the minimum is over the empty set. The same is true of every member of the corridor
family of \S\ref{sec:tier1c}, where waiting advances the corridor deterministically. Their
parameter therefore does \emph{not} vanish on every instance ours flags. Their $\mu_0$
compares two states of the \emph{true} model that the policy reaches, whereas
$\pi$-equivalence compares two \emph{candidate models}, so an ambiguity confined to a region
the policy never enters leaves Assumption~2 satisfied while \textsc{Pec} certifies
$\pi$-equivalence. On those instances the RegORL bound binds through its concentrability
condition instead. Conversely, $\mu_0$ vanishes whenever the policy cannot separate two states
of the true model, even when the candidate under test differs from the truth on an action the
policy takes, in which case \textsc{Pec} reports the pair as distinguishable. The two criteria
therefore overlap without either containing the other.

\paragraph{Deciding Assumption~2 with \textsc{Pec}.} Section~\ref{sec:regorl} shows that
$\mu_0 > 0$ is decided by running \textsc{Pec} from each same-layer pair and comparing the
separating horizon it reports against the number of steps the episode has left. We checked
this against the exact parameter on every pair of every environment and of the corridor
family, and the two agree everywhere. Table~\ref{tab:regorl-corridor} reports the cost on the
corridor family, where the ratio between the cost of the definition and that of \textsc{Pec}
grows from $3.7\times$ at depth 2 to $67.5\times$ at depth 7.

\begin{table}[!ht]
\centering\small
\caption{The corridor family: $\mu_0$ under the deviation, and the cost of deciding
Assumption~2 from the definition versus with \textsc{Pec}. Under $\pi$ the assumption holds
trivially at every depth while \textsc{Pec} certifies $\pi$-equivalence at every depth.}
\label{tab:regorl-corridor}
\begin{tabular}{rccrrr}
\toprule
depth $d$ & $\mu_0$ under $\pi'$ & ratio to $d-1$ & definition (s) & \textsc{Pec} (s) & speedup \\
\midrule
$2$ & $1/4$    & --      & $0.0005$ & $0.00014$ & $3.7\times$ \\
$3$ & $1/16$   & $0.250$ & $0.0024$ & $0.00054$ & $4.4\times$ \\
$4$ & $1/64$   & $0.250$ & $0.0100$ & $0.00143$ & $7.0\times$ \\
$5$ & $1/256$  & $0.250$ & $0.0408$ & $0.00314$ & $13.0\times$ \\
$6$ & $1/1024$ & $0.250$ & $0.1616$ & $0.00567$ & $28.5\times$ \\
$7$ & $1/4096$ & $0.250$ & $0.6536$ & $0.00969$ & $\mathbf{67.5\times}$ \\
\bottomrule
\end{tabular}
\end{table}

\paragraph{Decay of $\mu_0$ with depth.} On the corridor family, $\mu_0$ under the deviation
falls by a factor of exactly $4$ per room. Since $N_\delta \propto 1/\mu_0$, the corpus RegORL
needs grows like $4^d$ in the separating horizon even after the assumption holds. This is the
decay that motivated the formal-language metric of \citet{iclr2025rdp}, which
\S\ref{sec:intro} cites, and we exhibit it here as exact rationals on a family where depth is
the only thing varying. It is also the exponential in the separating horizon that
\S\ref{sec:method-depth} measures from the data side.

\subsection{\textsc{Pec} at scale}
\label{sec:app-scaling}

Corollary~\ref{cor:decide} states a running time linear in $|Q|\cdot|Q'|\cdot|A|\cdot|O|$
and independent of the horizon. Our environments are small because the method requires a known
pair (\S\ref{sec:experiments}), so we measure the bound on constructed pairs instead, sweeping
the product automaton from $10^2$ to $10^6$ state pairs.

\paragraph{Worst-case instances.} The bound describes the \emph{certifying} run. When
\textsc{Pec} returns a witness it stops at the first disagreeing pair and touches only the
part of the product that precedes it, whereas when it certifies $\pi$-equivalence it must
exhaust the reachable product. The sweep is therefore built from $\pi$-equivalent pairs, and
the shape of those pairs is forced. If every pair in $Q \times Q'$ is reachable and none
disagrees, then $e(q,a) = e'(q',a)$ for all $q,q'$ and every $a \in A_\pi$, so both emission
tables are constant on the actions the policy takes, since any variation would create a
disagreeing pair and hence an early return. A full product with no early exit thus forces flat
emissions under $A_\pi$, with the models differing on an action outside it. The corridor
construction of \S\ref{sec:tier1c} has the same form, and instances of this form are the worst
case for the procedure.

The instances are two cycles of coprime lengths $n$ and $m$ advanced by the policy's action.
From $(0,0)$ the product walks $(t \bmod n,\, t \bmod m)$, which by the Chinese remainder
theorem visits all $n\,m$ pairs, so the reachable product is exactly $|Q|\cdot|Q'|$.

\begin{table}[!ht]
\centering\small
\caption{\textsc{Pec} certifying $\pi$-equivalence. Every pair of the product is reachable, so
the procedure explores all of it before answering.}
\label{tab:scaling}
\begin{tabular}{rrrrrr}
\toprule
$|Q|$ & $|Q'|$ & product & pairs explored & seconds & \textmu s / pair \\
\midrule
$10$   & $11$   & $110$       & $110$       & $0.0005$ & $4.50$ \\
$50$   & $51$   & $2{,}550$   & $2{,}550$   & $0.0111$ & $4.36$ \\
$100$  & $101$  & $10{,}100$  & $10{,}100$  & $0.0438$ & $4.34$ \\
$200$  & $201$  & $40{,}200$  & $40{,}200$  & $0.1777$ & $4.42$ \\
$500$  & $501$  & $250{,}500$ & $250{,}500$ & $1.1263$ & $4.50$ \\
$700$  & $701$  & $490{,}700$ & $490{,}700$ & $2.2282$ & $4.54$ \\
$1000$ & $1001$ & $1{,}001{,}000$ & $1{,}001{,}000$ & $4.6457$ & $4.64$ \\
\bottomrule
\end{tabular}
\end{table}

The cost per pair is flat to within $8\%$ across four orders of magnitude, and the least-squares
slope of $\log(\text{time})$ against $\log(\text{pairs})$ over the seven sizes tabulated is
$1.003$, consistent with the linear growth in $|Q|\cdot|Q'|$ that the bound asserts at fixed
$|A|$ and $|O|$. Holding the product fixed at $40{,}200$ and enlarging the action support
instead gives $0.179$, $0.330$, $0.487$ and $0.604$ seconds for $|A_\pi| = 1,2,3,4$, which
isolates the $|A|$ factor.

\paragraph{Random transition tables.} Replacing the cycles with random transition tables gives
a reachable product between $38.7\%$ and $49.5\%$ of $|Q|\cdot|Q'|$, with a comparable
$4.5$--$5.2$\,\textmu s per pair and a log-log slope of $1.025$ against the pairs explored. At
$|Q| = 1000$ that is $419{,}192$ pairs in $2.17$ seconds. The linear scaling therefore comes
from the procedure and does not depend on the cycle construction.

\paragraph{Cost of finding a witness.} Fixing the product at
$1{,}001{,}000$ and moving only the depth at which the two models disagree, \textsc{Pec}
returns in $1.8\times10^{-5}$, $8.1\times10^{-5}$, $7.4\times10^{-4}$ and $7.2\times10^{-3}$
seconds at depths $0$, $10$, $100$ and $999$, against $4.65$ seconds to certify the same
product. Even at depth $999$ the search touches a thousandth of the pairs a certificate needs.
An instance whose corpus is insufficient therefore pays the full bound, while candidates that
separate early cost almost nothing.

\paragraph{Independence from the horizon.} None of these runs takes a horizon as input, and answers the infinite-horizon question. Enumerating histories answers only a finite-horizon
question, at a cost exponential in the horizon.

\section{Limitations and Open Problems}
\label{sec:limitations}

Almost-sure convergence is not formalised, since \S\ref{sec:martingale} supplies every
hypothesis Doob's theorem requires but the statement itself needs the trajectory measure. The
design procedure of \S\ref{sec:eig} is open-loop, and whether the lookahead bound of
Table~\ref{tab:lookahead} is tight for a closed-loop designer remains open. Scaling beyond the
five automaton states that exact enumeration handles requires approximate inference, which
this paper does not study. Designed experiments give no gain at short separating horizons
(\S\ref{sec:method-negative}), the reachability comparison holds only for the subclass
identified in \S\ref{sec:eig-scope}, an access-sequence heuristic narrows the gap in
separating mass between designed experiments and uniform exploration from a factor of $256$ to
a factor of about $8$ at depth eight (\S\ref{sec:app-eqtest}), and the freezing mechanism
itself is known in off-policy evaluation and causal inference. Continuous observations would
enter the account through aliasing that shrinks the separation $\gamma$ of
Appendix~\ref{sec:app-robust}, and the window of the tolerant test would close as $\gamma$
shrinks.

\end{document}